\documentclass{article}

\usepackage{microtype}
\usepackage{graphicx}
\usepackage{subfigure}
\usepackage{booktabs} 

\usepackage{amsmath}
\usepackage{amssymb}
\usepackage{amsthm}
\usepackage{amsfonts}
\usepackage{scalerel}
\usepackage{natbib}
\usepackage{url}
\usepackage{lstbayes}

\usepackage{hyperref}

\newcommand{\E}{{\mathbb E}}

\newcommand{\Var}{{\mathbb V}{\rm ar}}

\newcommand{\RR}{{\mathbb R}}

\newcommand{\Pro}{{\mathbb P}}

\newcommand{\n}{\|}
\newcommand{\bi}{\begin{itemize}}
\newcommand{\ei}{\end{itemize}}
\newcommand{\be}{\begin{enumerate}}
\newcommand{\ee}{\end{enumerate}}
\newcommand{\ra}{\rightarrow}

\newcommand{\ep}{\epsilon}

\newcommand{\iy}{\infty}

\newcommand{\beq}{\begin{equation}}
\newcommand{\eeq}{\end{equation}}
\newcommand{\beqa}{\begin{eqnarray*}}
\newcommand{\eeqa}{\end{eqnarray*}}
\newcommand{\btm}{\begin{theorem}}
\newcommand{\etm}{\end{theorem}}
\newcommand{\bpf}{\begin{proof}}
\newcommand{\epf}{\end{proof}}
\newcommand{\bla}{\begin{lemma}}
\newcommand{\ela}{\end{lemma}}
\newcommand{\bdn}{\begin{definition}}
\newcommand{\edn}{\end{definition}}
\newcommand{\bpn}{\begin{proposition}}
\newcommand{\epn}{\end{proposition}}
\newcommand{\bcy}{\begin{corollary}}
\newcommand{\ecy}{\end{corollary}}
\DeclareMathOperator*{\foo}{\scalerel*{+}{\sum}}

\usepackage[accepted]{icml2018}

\icmltitlerunning{Bayesian LOO for large data}

\newtheorem{proposition}{Proposition}
\newtheorem{definition}{Definition}
\newtheorem{lemma}[proposition]{Lemma}
\newtheorem{corollary}[proposition]{Corollary}

\begin{document}

\twocolumn[
\icmltitle{Bayesian leave-one-out cross-validation for large data}



\icmlsetsymbol{equal}{*}

\begin{icmlauthorlist}
\icmlauthor{M\r{a}ns Magnusson}{aalto}
\icmlauthor{Michael Riis Andersen}{aalto,dtu}
\icmlauthor{Johan Jonasson}{chalmers}
\icmlauthor{Aki Vehtari}{aalto}
\end{icmlauthorlist}

\icmlaffiliation{aalto}{Department of Computer Science, Aalto University, Finland}
\icmlaffiliation{chalmers}{Department of Mathematical Sciences, Chalmers University of Technology and University of Gothenburg, Sweden \footnote{Research partly supported by WASP AI/Math}}
\icmlaffiliation{dtu}{Department of Applied Mathematics and Computer Science, Technical University of Denmark, Denmark}

\icmlcorrespondingauthor{M\r{a}ns Magnusson}{mans.magnusson@aalto.fi}

\icmlkeywords{Approximate inference, Leave-one-out cross-validation, Model inference, Subsampling}





]



\printAffiliationsAndNotice{} 

\begin{abstract}
Model inference, such as model comparison, model checking, and model selection, is an important part of model development. Leave-one-out cross-validation (LOO) is a general approach for assessing the generalizability of a model, but unfortunately, LOO does not scale well to large datasets. We propose a combination of using approximate inference techniques and probability-proportional-to-size-sampling (PPS) for fast LOO model evaluation for large datasets. We provide both theoretical and empirical results showing good properties for large data.
\end{abstract}

\section{Introduction}
\label{sec:intro}

Model inference, such as model comparison, checking, and selection, is an integral part of developing new models. From a Bayesian decision-theoretic point of view (see \citet{vehtari2012survey} for an extended discussion) we want to make a choice $a \in \mathcal{A}$, in our case a model $p_M$, that maximize our \emph{expected utility} for a utility function $u(a, \cdot)$ as
\[
\bar{u}(a) = \int u(a,\tilde{y_i}) p_t(\tilde{y}_i) d\tilde{y}_i\,,
\]
where $p_t(\tilde{y}_i)$ is the true probability distribution generating observation $\tilde{y}_i$.

A common scenario is to study how well a model \emph{generalizes} to unseen data \citep{box1976science,vehtari2012survey,vehtari2017practical}. A popular utility function $u$ with good theoretical properties for probabilistic models is the log score function \cite{bernardo1979expected,robert1996intrinsic,vehtari2012survey}. The log score function give rise to using the \emph{expected log predictive density} (elpd) for model inference, defined as
\[
\overline{\text{elpd}}_M= \int \log p_M(\tilde{y}_i|y) p_t(\tilde{y}_i) d\tilde{y}_i\,,
\]
where $\log p_M(\tilde{y}_i|y)$ is the log predictive density for a new observation for the model $M$. 

Leave-one-out cross-validation (LOO-CV) is one approach to estimate the elpd for a given model, and is the method of focus in this paper \cite{bernardo1994bayesian,vehtari2012survey,vehtari2017practical}. Using LOO-CV we can treat our observations as pseudo-Monte Carlo samples from $p_t(\tilde{y}_i)$ and estimate the $\overline{\text{elpd}}_\text{loo}$ as 
\begin{align}
\overline{\text{elpd}}_\text{loo} & = \frac{1}{n} \sum^n_{i=1} \log p_M(y_i|y_{-i}) \label{elpd_loo} \\ 
& = \frac{1}{n} \sum^n_{i=1} \log \int p_M(y_i|\theta) p_M(\theta | y_{-i}) d\theta \nonumber \\ 
& = \frac{1}{n} \, \text{elpd}_\text{loo}\,, \nonumber
\end{align}
where $n$ is the number of observations (that may be very large), $p_M(y_i | \theta)$ is the likelihood, and $p_M(\theta | y_{-i})$ is the posterior for $\theta$ where we hold out observation $i$. This will henceforth be called the LOO posterior and $p_M(\theta | y)$ will be referred to as the full posterior. In this paper both $\overline{\text{elpd}}_\text{loo}$ and $\text{elpd}_\text{loo}$ will be quantities of interest, depending on the situation.

Bayesian LOO-CV has many appealing theoretical properties compared to other common model evaluation techniques. The popular $k$-fold cross-validation is, in general, a biased estimator of $\text{elpd}_M$, since each model is only trained using a subset of the full data \cite{vehtari2012survey}. The LOO-CV is, just as the Watanabe-Akaike Information criteria (WAIC), a consistent estimate of the true $\text{elpd}_M$ for regular and singular models \cite{watanabe2010asymptotic}. A model is regular if the map taking the parameters to the probability distribution is one-to-one and the Fisher information is positive-definitive. If a model is not regular, then the model is singular \cite{watanabe2010asymptotic}. Since many models, such as neural networks, normal mixture models, hidden Markov models, and topic models, are singular, we need consistent methods to estimate the $\text{elpd}$ for singular models \cite{watanabe2010asymptotic}. Although the WAIC and LOO-CV have the same asymptotic properties, recent research has shown that the LOO-CV is more robust than WAIC in the finite data domain \cite{vehtari2017practical}.

In addition to the theoretical properties, the LOO-CV also gives an intuitive framework for evaluating models where the user easily can use different utility functions of interest as well as easily taking hierarchical data structures into account by using leave-one-group-out or leave-one-cluster out cross-validation (see \citet{merkle2018bayesian} for a discussion). Taken together, LOO-CV has many very good properties, both empirical and theoretical. In this paper, we will focus on LOO-CV as a way of evaluating models.

Modern probabilistic machine learning techniques need to scale to massive data. In a data-rich regimes, we often want complex models, such as hierarchical and non-linear models. Model comparison and model evaluation are important for model development, but little focus has been put into finding ways of scaling LOO-CV to larger data. The main problem is that a straight-forward implementation means that $n$ models need to be estimated. Even if this problem is solved, for example using importance sampling (see below), we still have two problems.


First (1), many posterior approximation techniques, such as Markov Chain Monte Carlo (MCMC), does not generally scale to large $n$ or is computationally very costly. Second (2), computing $\text{elpd}_\text{loo}$ still needs to be computed over $n$ observations. If it is costly to estimate individual contributions (i.e. $\log p_M(y_i|y_{-i})$), computing the total $\text{elpd}_\text{loo}$ may be very costly for very large models.

\subsection{Pareto-smoothed importance sampling}

If we would implement LOO-CV naively, inference needs to be repeated $n$ times for each model. \citet{gelfand1996model} propose the use of importance sampling to solve this problem. The idea is to estimate $p_M(y_i|y_{-i})$ in Eq. \eqref{elpd_loo} using the importance sampling approximation 
\begin{align}
\label{eq:is_estimate}
\log \hat{p}(y_i|y_{-i}) = \log\left( \frac{\frac{1}{S} \sum_{s=1}^S p_M(y_i|\theta_s) r(\theta_s)}{\frac{1}{S} \sum_{s=1}^S r(\theta_s)} \right), 
\end{align}
where $\theta_s$ are $s \in 1,...,S$ draws from the full posterior $p(\theta|y)$, and 
\begin{align}
r(\theta_s) & = 
\frac{p_M(\theta_s|y_{-i})}{p_M(\theta_s|y)} \nonumber \\
& \propto \frac{1}{p_M(y_i|\theta_s)}  \nonumber\,,
\end{align}
where the last step is a well-known result of \citet{gelfand1996model}. The ratios $r(\theta_s)$ can be unstable due to a long right tail, but this can be resolved using Pareto-smoothed importance sampling (PSIS) \citep{vehtari2015pareto}. Using PSIS we fit a generalized Pareto distribution to the largest weights $r(\theta_s)$ and replace the largest importance sample ratios with order statistics from the estimated generalized Pareto distribution, decreasing the variance by introducing a small bias. PSIS also has the benefit that we can use the estimated shape parameter $\hat{k}$ from the generalized Pareto distribution to determine the reliability of the estimate. For data-points with $\hat{k}>0.7$ the estimates of $\log p(y_i|y_{-i})$ \emph{can} be unreliable and hence $\hat{k}$ can be used as a diagnostic \citep{vehtari2017practical}.

However, PSIS-LOO has the same scaling problem as LOO-CV in general since it requires (1) samples from the true posterior (e.g. using MCMC) and (2) the estimation of the $\text{elpd}_\text{loo}$ contributions from all observations \citep{gelfand1996model,vehtari2017practical}. Both of these requirements can be costly in a data-rich regime and are the main problems we address in this paper.

\subsection{Contributions and limitations}
In this paper, we focus on the problems of LOO-CV for large datasets and our contributions are three-fold. First, we extend the method of \citet{gelfand1996model} to posterior approximations by including a correction term to the importance sampling weights. Second, we propose sampling individual $\text{elpd}_\text{loo}$ components with probability-proportional-to-size sampling (PPS) to estimate $\text{elpd}_\text{loo}$. Third, we show theoretically that these contributions have very favorable asymptotic properties as $n \rightarrow \infty$. We show that the proposed estimator for $\text{elpd}_\text{loo}$ is consistent for any consistent posterior approximation $q$ (such as Laplace approximations, mean-field, and full-rank variational inference posterior approximations). We also show that the variance due to subsampling will decrease as the number of observations $n$ grows. In the limit, and given the assumptions in Section \ref{subsec:theoretical}, we only need one subsampled observation, and one draw from the full posterior, to estimate $\overline{\text{elpd}}_\text{loo}$ with zero variance. Taken together this introduces a new, fast, and theoretically motivated approach to  model evaluation for large datasets.

The limitations of our approach are the same as using PSIS-LOO \citep{vehtari2017practical} as well as the requirement that the approximate posterior needs to be sufficiently close to the true posterior \citep[see][for a discussion]{yao18a}.

\section{Bayesian leave-one-out cross-validation for large data sets}
\label{sec:methods}

Leave-one-out cross-validation (LOO-CV) has very good theoretical and practical properties. This makes it relevant to develop tools to scale LOO-CV. We solve this problem using scalable posterior approximations, such as Laplace and variational approximations and using probability-proportional-to-log-predictive-density subsampling inspired by \citet{hansen1943}.

\subsection{Estimating the elpd using posterior approximations}

Laplace and variational posterior approximations are attractive for fast model comparisons due to their computational scalability.  Laplace approximation approximate the posterior distribution with multivariate normal distribution $q_{Lap}(\theta|y)$ with the mean being the mode of the posterior and the covariance the inverse Hessian at the mode \cite{azevedo1994laplace}.

In variational inference, we  minimize the Kullback-Leibler (KL) divergence between an approximate family $\mathcal{Q}$ of densities and the true posterior $p(\theta|y)$ \cite{jordan1999introduction,blei2017variational}. Hence we find the approximation $q$ that is closest to the true posterior in a KL divergence sense. Here we let $\mathcal{Q}$ be a family of multivariate normal distributions with a diagonal covariance structure (mean-field) or a full covariance structure (full-rank). Hence we will work with a mean-field variational approximation $q_{MF}(\theta|y)$ and a full-rank variational approximation $q_{FR}(\theta|y)$.

Although, all these posterior approximations, $q_{Lap}(\theta|y), q_{MF}(\theta|y)$, and $ q_{FR}(\theta|y)$, will, in general, be different than the true posterior distribution, we can use them as a proposal distribution in an importance sampling scheme. In this scheme we use a posterior approximation $q_M(\theta|y)$ for a model $M$ as the proposal distribution and $p_M(\theta|y_{-i})$, the LOO posterior, as our target distribution. The expectation of interest is the same as in the standard PSIS-LOO given by Eq. \eqref{eq:is_estimate}, but we also propose to correct for the posterior approximation error. Hence we change $r(\theta)$ to
\begin{align}
\label{eq:adviloo_weights}
r(\theta_s) & = 
\frac{p_M(\theta_s|y_{-i})}{q_M(\theta_s|y)} \nonumber \\
& = \frac{p_M(\theta_s|y_{-i})}{p_M(\theta_s|y)} \frac{p_M(\theta_s|y)}{q_M(\theta_s|y)} \\ 
& \propto \frac{1}{p_M(y_i|\theta_s)} \frac{p_M(\theta_s|y)}{q_M(\theta_s|y)} \nonumber\,.
\end{align}

The two-factor approach is needed to take the posterior approximation into account. The factorization in Eq. \eqref{eq:adviloo_weights} shows that the importance correction contains two parts, the correction from the full posterior to the LOO posterior and the correction from the full approximate distribution to the full posterior. Both components often have lighter tailed proposal distribution than the corresponding target distribution which can increase the variance of the importance sampling estimate \citep{geweke1989,gelfand1996model}.

Pareto-smoothed importance sampling can be used to both stabilize the weights in estimating the contributions to the $\text{elpd}_\text{loo}$ and in evaluating variational inference approximations using $\hat{k}$ as a diagnostic \citep{vehtari2015pareto,yao18a}. Hence we use PSIS to stabilize the weights with the additional benefit that we can use $\hat{k}$, the shape parameter in the generalized Pareto distribution, to diagnose how well the approximation is working \citep{vehtari2015pareto}.

\subsection{Probability-proportional-to-size subsampling and Hansen-Hurwitz estimation}

Using PSIS we can estimate each $\log \hat{p}(y_i|y_{-i})$ term and sum them to estimate $\text{elpd}_{\text{loo}}$. Estimating every $\log \hat{p}(y_i|y_{-i})$ can be costly, especially as $n$ grows. In some situations using PSIS-LOO, estimating $\text{elpd}_{\text{loo}}$ can take even longer than computing the full posterior once, due to the computational burden of computing $\log \hat{p}(y_i|y_{-i})$, estimating $\hat{k}$ and using the generalized Pareto distribution to stabalize the weights for each individual observation. To handle this problem we suggest using a sample of the $\text{elpd}_{\text{loo}}$ components to estimate $\text{elpd}_{\text{loo}}$. 

Estimating totals, such as $\text{elpd}_{\text{loo}}$, has a long tradition in sampling theory \citep[see][]{cochran77}. If we have auxiliary variables that are a good approximation of our variable of interest, we can use a probability-proportional-to-size (PPS) sampling scheme to reduce the sampling variance in the estimate of $\text{elpd}_{\text{loo}}$ using the \emph{unbiased} Hansen-Hurwitz (HH) estimator \citep{hansen1943}. When evaluating models, we can often easily compute $\log p_M(y_i|y)$, the full posterior log predictive density, for all observations. We then sample $m<n$ observations proportional to $\tilde{\pi}_i \propto \pi_i = -\log p_M(y_i|y) = -\log \int p_M(y_i|\theta) p_M(\theta|y) d\theta$. We here assume that all $\log p_M(y_i|y) < 0$, but this assumption is only for convenience. 

In the case of regular models and large $n$, we can also approximate $\log p_M(y_i|y) \approx \log p_M(y_i|\hat{\theta})$ where $\hat{\theta}$ can be a Laplace posterior mean estimate $\hat{\theta}_q$ or a VI mean estimate $\E_{\theta \sim q}[\theta]$. In the case of VI and Laplace approximations, this further speeds up the computation of the $\tilde{\pi}_i$s since we do not need to integrate over $\theta$ for all $n$ observations. Using a sampling with probability-proportional-to-size scheme, the estimator for $\text{elpd}_\text{loo}$ can be formulated as 

\begin{align}
\label{eq:hh_estimate}
\widehat{\overline{\text{elpd}}}_{\text{loo},q} = \frac{1}{n} \frac{1}{m} \sum^m_i  \frac{1}{\tilde{\pi}_i} \log \hat{p}(y_i|y_{-i})\,, 
\end{align}
where $\tilde{\pi}_i$ is the probability of subsampling observation $i$, $\log \hat{p}(y_i|y_{-i})$ is the (self-normalized) importance sampling estimate of $\log p(y_i|y_{-i})$ given by Eq. \eqref{eq:is_estimate} and \eqref{eq:adviloo_weights}, and $m$ is the subsample size. The variance estimator can be expressed as \citep[see][Theorem 9A.2.]{cochran77}.

\begin{align}
\label{eq:hh_var_estimate}
& v(\widehat{\overline{\text{elpd}}}_{\text{loo},q}) = \nonumber\\
& \frac{1}{n^2 m (m-1)}\sum^m_{i=1} \left(\frac{\log\hat{p}(y_i|y_{-i})}{\tilde{\pi}_i}-n \widehat{\overline{\text{elpd}}}_{\text{loo}}\right)^2\,.
\end{align}

The benefits of the HH estimator are many. First, if the probabilities are proportional to the variable of interest ($\log\hat{p}(y_i|y_{-i})$ here), the variance in Eq. \eqref{eq:hh_var_estimate} will go to zero, a property of use in the asymptotic analysis in Section \ref{subsec:theoretical}. Second, the estimator of $\overline{\text{elpd}}_\text{loo}$ is not limited to posterior approximation methods, but can also be used with MCMC (but without the importance sampling correction factor). Third, PPS sampling has the benefit that we can use Walker-Alias multinomial sampling \citep{walker1977efficient}. By building an Alias table in $O(n)$ time we can then sample a new observation in $O(1)$ time. This means that can continue to sample observations until we have sufficient precision in $\overline{\text{elpd}}_\text{loo}$ for our model comparison purposes, independent of the number of observations $n$. Fourth, the estimator is unbiased for all $\tilde{\pi}_i$. So by using $\log p(y_i|\hat{\theta})$ instead of $\log p(y_i|y)$ we would expect a small increase in variance since we would expect that for finite $n$, $\log p(y_i|y)$ would be a better approximation of $\log p(y_i|y_{-i})$ than $\log p(y_i|\hat{\theta})$, but at a greater computational cost.

To compare models, we are often also interested in the variance of $\overline{\text{elpd}}_\text{loo}$, or for the dataset, henceforth called $\sigma^2_\text{loo}$. To estimate $\sigma^2_\text{loo}$ we can use the same observations as sampled previously, as
\begin{align}
\label{eq:hh_se_estimate}
\hat{\sigma}^2_\text{loo} & = \frac{1}{n m} \sum^m_i  \frac{\hat{p}^2_i}{\tilde{\pi}_i} + \\
& \frac{1}{n^2 m (m - 1)} \sum^m_i \left( \frac{\hat{p}_i}{\tilde{\pi}_i} - \frac{1}{m} \sum^m_i \frac{\hat{p}_i}{\tilde{\pi}_i} \right)^2 - \nonumber\\
& \left(\frac{1}{n m} \sum^m_i \frac{\hat{p}_i}{\tilde{\pi}_i} \right)^2 \nonumber
\end{align} 
where $\hat{p}_i = \log \hat{p}(y_i|y_{-i})$.  For a proof of unbiasedness of the $\hat{\sigma}^2_\text{loo}$ estimator for $\sigma^2_\text{loo}$ in Eq. \eqref{eq:hh_se_estimate}, see the supplementary material. Also, note that here  $\sigma^2_\text{loo} = \frac{1}{n}\sum_i^n (\hat{p}^2_i - (\frac{1}{n}\sum_i^n \hat{p}^2_i)^2)$, which in itself is not an unbiased estimate for the true $\sigma^2_\text{loo}$ \cite{bengio2004no}.

Although the variance estimator is unbiased, it is not as efficient as the estimator of $\text{elpd}_\text{loo}$. This is part due to the fact that $\tilde{\pi}_i$ is not proportional to $\hat{p}^2_i$ in the first line in Eq. \eqref{eq:hh_se_estimate}. This can be solved by sampling in two steps both proportional to $\hat{p}_i$ and $\hat{p}^2_i$.

\subsection{Asymptotic properties}
\label{subsec:theoretical}

For larger data sets the asymptotic properties of the method are crucial and we derive asymptotic properties for the methods as follows. We consider a generic Bayesian model; a sample $(y_1,y_2,\ldots,y_n)$, $y_i \in \mathcal{Y} \subseteq \RR$, is drawn from a true density $p_t = p(\cdot|\theta_0)$ for some true parameter $\theta_0$. The parameter $\theta_0$ is assumed to be drawn from a prior $p(\theta)$ on the parameter space $\Theta$, which we assume to be an open and bounded subset of $\RR^d$. A number of conditions are used. They are as follows.

\begin{itemize}
  \item[(i)] the likelihood $p(y|\theta)$ satisfies that there is a function $C:\mathcal{Y} \ra \RR_+$, such that $\E_{y \sim p_t}[C(y)^2] < \iy$ and such that for all $\theta_1$ and $\theta_2$, $|p(y|\theta_1)-p(y|\theta_2)| \leq C(y)p(y|\theta_2)\n\theta_1-\theta_2\n$.
   \item[(ii)] $p(y|\theta)>0$ for all $(y,\theta) \in \mathcal{Y} \times \Theta$,
   \item[(iii)] There is a constant $M<\iy$ such that $p(y|\theta) < M$ for all $(y,\theta)$,
  \item[(iv)] all assumptions needed in the Bernstein-von Mises (BvM) Theorem \cite{walker1969asymptotic},
  \item[(v)] for all $\theta$, $\int_{\mathcal{Y}}(-\log p(y|\theta))p(y|\theta)dy < \iy$.
\end{itemize}

Of these assumptions, (i) and (iv) are the most restrictive. The assumption that the parameter space is bounded is not very restrictive in practice since we can approximate any proper prior arbitrarily well with a  truncated approximation.



\begin{proposition} \label{pa}
Let the subsampling size $m$ and the number of posterior draws $S$ be fixed at arbitrary integer numbers, let the sample size $n$ grow, assume that (i)-(v) hold and let $q=q_n(\cdot|y)$ be any consistent approximate posterior. Write $\hat{\theta}_q = \arg\max\{q(\theta): \theta \in \Theta\}$ and assume further that $\hat{\theta}_q$ is a consistent estimator of $\theta_0$. Then

\[|\widehat{\overline{\text{elpd}}}_\text{loo}(m, q)-\overline{\text{elpd}}_{loo}| \ra 0\]
in probability as $n \ra \iy$ for any of the following choices of $\pi_i$, $i=1,\ldots,n$.

\begin{itemize}
  \item[(a)] $\pi_i= -\log p(y_i|y)$,
  \item[(b)] $\pi_i= - \E_y[\log p(y_i|y)]$,
  \item[(c)] $\pi_i= - \E_{\theta \sim q}[\log p(y_i|\theta)]$,
  \item[(d)] $\pi_i = -\log p(y_i|\E_{\theta \sim q}[\theta])$,
  \item[(e)] $\pi_i = -\log p(y_i|\hat{\theta}_q)$.
\end{itemize}

\end{proposition}


\begin{proof}
See the supplementary material.
\end{proof}

This proposition has three main points. First, the estimator of the $\widehat{\overline{\text{elpd}}}_\text{loo}$ is consistent for any consistent posterior approximation. In the limit, the mean-field variational approximation will also estimate the true $\overline{\text{elpd}}_\text{loo}$. Second, the estimator is also consistent irrespective of the sub-sampling size $m$ and the number of draws, $S$, from the posterior. This is a very good scaling characteristic. Third, the estimator is consistent also if we approximate $\tilde{\pi}_i$ with $\log p(y_i | \hat{\theta}_{q})$. This means that for larger data we can plug in point estimates to quickly compute $\tilde{\pi}_i$ and still have the consistency property.

The main limitations with Proposition \ref{pa} are that it is based on the consistency of the posterior approximations and the proposition does only hold for regular models for which $q$ are consistent. This is mainly due to the fact that Laplace and VI are not, in general, consistent for singular models.

\subsection{Computational complexity}
\label{subsec:complexity}

In the large $n$ domain it is also of interest to study the computational complexity of our approach. Assuming that the additional cost of computing $p(y_i|y_{-i})$ compared to the point log predictive density (lpd) at $\hat{\theta}$, $\log p(y_i|\hat{\theta})$, is $O(S)$, where $S$ is the number of samples from the full posterior. Then the cost of computing the full $\text{elpd}_\text{loo}$ is 
\[
O(nS)\,.
\]
If we instead use our proposed method we would have the complexity
\[
O(n + mS)\,,
\]
where $m$ is the subsampling size. Using the proposed approach, we get an \emph{unbiased} estimate of $\text{elpd}_\text{loo}$ together with the variance $v(\widehat{\text{elpd}}_\text{loo})$ of that estimate, giving us information on the precision of the method for a given $m$.

Finally, we could, for large $n$ just use the same lpd as an approximation with complexity
\[
O(n)\,.
\]
This estimate is though \emph{biased} for all finite $n$, and we have no diagnostic indicating how good or bad the approximation is.

This shows the large-scale characteristic of our proposed approach. By adding a small cost ($mS$), we will have a good estimate of the true $\text{elpd}_\text{loo}$ at the same cost as computing just the lpd. If using the point lpd is a good approximation we would need less $m$. On the other hand, if the point lpd would be a bad approximation, we would need a larger $m$. The variance estimator in Eq. \eqref{eq:hh_var_estimate} would in these situations serve as an indicator, with a higher variance estimate.

\subsection{Method summary}

We have presented a method for estimating the $\text{elpd}$ efficiently using posterior approximation and PPS subsampling. One of the attractive properties of the method is that we can diagnose if the method is working. Using PSIS-LOO we can diagnose the estimation of each individual $\log \hat{p}(y_i|y_{-i})$ as well as the overall posterior approximation using the $\hat{k}$ diagnostic. Then the variance of the HH estimator in Eq. \eqref{eq:hh_var_estimate} captures the effect of the subsampling in the finite $n$ case. Our approach for large-scale LOO can be summarized in the following steps. 

\begin{enumerate}
    \item Estimate the models of interest using any consistent posterior approximation technique.
    \item Compute the $\hat{k}$ diagnostic for the posterior to asses the general overall posterior approximation. See \cite{yao18a} for an example for variational inference.
    \item Compute $\tilde{\pi}_i \propto -\log p(y_i|y)$ for all $n$ observations. For regular models this can be approximated with $\tilde{\pi}_i \propto-\log p(y_i|\hat{\theta})$ for large data.
    \item Sample $m$ observations using PPS sampling and compute $\log \hat{p}(y_i|y_{-i})$ using Eq. \eqref{eq:is_estimate} and \eqref{eq:adviloo_weights} for the sampled observations. Use $\hat{k}$ to diagnose the estimation of each individual $\log \hat{p}(y_i|y_{-i})$. 
    \item Estimate  $\widehat{\text{elpd}}_{\text{loo}}$, $v(\widehat{\text{elpd}}_{\text{loo}})$, and $\hat{\sigma}^2_\text{loo}$ using Eq. \eqref{eq:hh_estimate}, \eqref{eq:hh_var_estimate}, and \eqref{eq:hh_se_estimate} to compare model predictive performance. 
    \item Repeat step 3 and 4 until sufficient precision is reached.
\end{enumerate}

The downside is that the $\hat{k}$ diagnostic can be too conservative for our purpose. In the case of a correlated posterior and mean-field variational inference, $\hat{k}$ may indicate a poor approximation even though the estimation of $\text{elpd}_\text{loo}$ is still consistent and may work well. In this situation, we would get a result indicating that all $\log \hat{p}(y_i|y_{-i})$ are problematic, even though the estimation actually work well, something we will see in the experiments.

\section{Experiments}
\label{sec:experiments}

To study the characteristic of the proposed approach we study multiple models and datasets. We use simulated datasets used to fit a Bayesian linear regression model with $D$ variables and $N$ observations. The data is generated such that so we get either a correlated (c) or an independent (i) posterior for the regression parameters by construction. This will enable us to study the effect of the mean-field assumptions in variational posterior approximations. 
In addition, we use data from the radon example of \citet{lin1999analysis} to show performance on a larger dataset with multiple models. 

All posterior computations uses Stan 2.18 \cite{carpenter2017stan,standev2018stancore} and all models used can be found in the supplementary material. The methods has been implemented using the \texttt{loo} R package \cite{vehtari2018loo} framework for Stan and is available as supplementary material. We use mean-field and full-rank Automatic Differentiation Variational Inference (ADVI) \citep{kucukelbir2017automatic} and Laplace approximations as implemented in Stan. ADVI automatically handles constrained variables and uses stochastic variational inference.

\subsection{Estimating $\text{elpd}_\text{loo}$ using posterior approximations}

\begin{table*}[ht]
\centering
\begin{tabular}{llrrrr}
  \toprule
Data &  & ADVI(FR) & ADVI(MF) & Laplace & MCMC \\ 
  \midrule
LR(c) 100D & $\text{elpd}_\text{loo}$ & -14249 & -14267 & -14247 & -14247 \\ 
   & $\hat{k}>0.7$ (\%) & 100 & 100 & 0 & 0 \\ 
  \midrule
  LR(c) 10D & $\text{elpd}_\text{loo}$ & -14271 & -14271 & -14272 & -14272 \\ 
   & $\hat{k}>0.7$ (\%) & 0 & 100 & 0 & 0 \\ 
  \midrule
  LR(i) 100D &  $\text{elpd}_\text{loo}$ & -14193 & -14239 & -14238 & -14239 \\ 
   & $\hat{k}>0.7$ (\%) & 100 & 0 & 0 & 0 \\ 
  \midrule
  LR(i) 10D &  $\text{elpd}_\text{loo}$ & -14202 & -14202 & -14202 & -14203 \\ 
   & $\hat{k}>0.7$ (\%) & 0 & 0 & 0 & 0 \\ 
   \bottomrule
\end{tabular}
\caption{Estimation of $\text{elpd}_\text{loo}$ using posterior approximations. For all models and posterior approximations, $\sigma_\text{LOO} \approx 70$. No subsampling is used and MCMC is gold standard.} 
\label{tab:elpd}
\end{table*}

Table \ref{tab:elpd} contains the estimatied values of $\text{elpd}_\text{loo}$ for different posterior approximations. We used we used 100~000 iterations for ADVI and 1000 warmup iterations and 2000 samples from 2 chains for the MCMC. From the table, we can see that using PSIS-LOO and posterior approximations to estimate $\text{elpd}_\text{loo}$ works well or diagnostic correctly indicates the failure. As we would expect, the mean-field approximation for the correlated posterior does not approximate the true posterior very well when the posterior has correlated parameters and the $\hat{k}$ values are too high for all observations. In spite of the high $\hat{k}$ values, the estimate of the $\text{elpd}_\text{loo}$ is not very far from the (gold-standard) MCMC estimate, showing the consistency result in Prop. \ref{pa} for mean-field approximations - even when the true posterior covariance structure is not in the variational family.

The second result is that the full-rank VI approximation has a poor fit for a large number of parameters ($D = 100$). This comes from that the full rank ADVI needs to approximate the full posterior covariance structure (with ~5~000 parameters) based on stochastic gradients. The increased perturbation in the estimate of the covariance matrix has the effect of increasing the overall $\hat{k}$, especially for larger dimensions indicating a less good approximation of the posterior. 

\subsection{Subsampling using PPS sampling}

\begin{table}[ht]
\centering
\begin{tabular}{lllrrr}
  \toprule
Data & $m$ & Method & $\widehat{\text{elpd}}_\text{loo}$ & SE($\widehat{\text{elpd}}_\text{loo}$) & $\hat{\sigma}_\text{loo}$ \\ 
  \midrule
LR(c) & - & \textbf{True} & \textbf{-14247} & \textbf{0} & \textbf{71} \\ 
  100D & 10 & PPS(1) & -14236 & 11.8 & 23 \\ 
   &  & PPS(2) & -14244 & 13.6 & 67 \\ 
   &  & SRS & -14234 & 2197.8 & 70 \\ \cline{2-6}
   & 100 & PPS(1) & -14248 & 4.3 & 67 \\ 
   &  & PPS(2) & -14249 & 4.5 & 73 \\ 
   &  & SRS & -13823 & 598.8 & 60 \\ \cline{2-6}
   & 1000 & PPS(1) & -14245 & 1.5 & 69 \\ 
   &  & PPS(2) & -14248 & 1.5 & 71 \\ 
   &  & SRS & -14068 & 212.9 & 67 \\ 
  \midrule   
  LR(c) & - & \textbf{True} & \textbf{-14272} & \textbf{0} & \textbf{71} \\ 
  10D & 10 & PPS(1) & -14272 & 3.2 & 88 \\ 
   &  & PPS(2) & -14269 & 3.6 & 81 \\ 
   &  & SRS & -18096 & 3310.2 & 105 \\ \cline{2-6}
   & 100 & PPS(1) & -14272 & 1.2 & 87 \\ 
   &  & PPS(2) & -14272 & 1.1 & 78 \\ 
   &  & SRS & -13921 & 669.9 & 67 \\ \cline{2-6}
   & 1000 & PPS(1) & -14272 & 0.4 & 75 \\ 
   &  & PPS(2) & -14272 & 0.4 & 69 \\ 
   &  & SRS & -14266 & 223.3 & 71 \\ 
   \bottomrule
\end{tabular}
\caption{Effect of subsampling proportional to log predictive density. The result are based on MCMC draws and $\hat{\theta}$ is the posterior mean for the parameters. PPS(1) is subsampling proportional to $-\log(p(y_i|y))$, PPS(2) is subsampling proportional to $-\log(p(y_i|\hat{\theta}))$, and SRS is simple random sampling.} 
\label{tab:hh}
\end{table}

Table \ref{tab:hh} shows empirical results on the effect of using subsampling proportional to the predictive density compared to simple random sampling. The results are much in line with what we would expect from the theory presented in Section \ref{subsec:theoretical}. We can see that the proposed method, sampling proportional to $-\log(p(y_i|y))$ (PPS(1)) and sampling proportional to $-\log(p(y_i|\hat{\theta}))$ (PPS(2)) outperforms simple random sampling with orders of magnitude. Using just a sample size of $m=10$ observations using our proposed method is much more precise than using $m=1000$ observations with simple random sampling, although we can see that the estimate $\hat{\sigma}_\text{loo}$ is less reliable for such small sample sizes. This results can be explained by the sampling probabilities used in the subsampling procedure. Figure \ref{fig:ppspi} show the distribution of sampling probabilities where we can see that the probabilities are highly skewed, indicating the reason for the inefficiency of the SRS compared to the proposed approach. Table \ref{tab:hh} also show that sampling with $\tilde{\pi}\propto-\log(p(y_i|\hat{\theta}))$ does not cost us very much in precision when estimating $\text{elpd}_\text{loo}$. In many situations with larger data we would expect that using a point estimate, $\hat{\theta}$, of the parameters in computing the likelihood values would be much faster than computing $-\log(p(y_i|y))$. 
  
\begin{figure}[htbp]
  \label{fig:ppspi}
  {\includegraphics[width=1\linewidth]{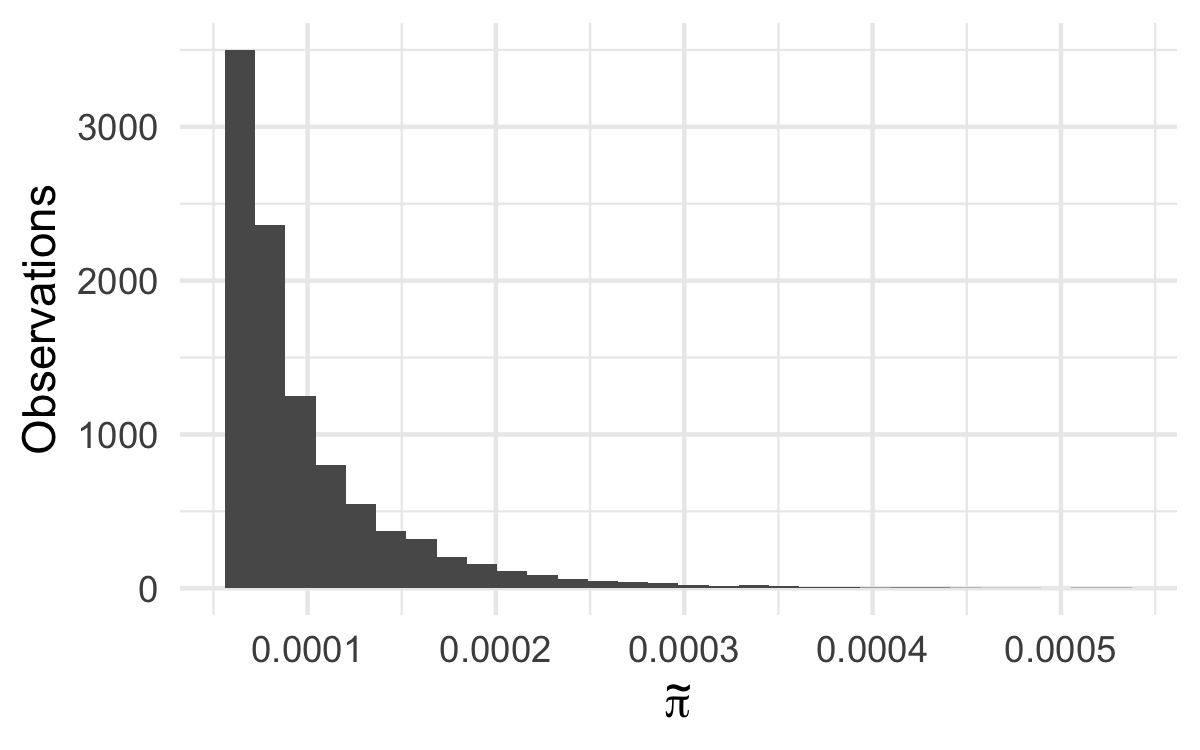}}
  {\caption{Sampling probabilities ($\tilde{\pi}$) for the LR(cor) 100D data and $\pi_i = -\log p(y_i|y)$. The results are very similar for the other LR data and for $\pi_i = -\log p(y_i|\hat{\theta})$.}}
\end{figure}

\begin{table}[ht]
\centering
\begin{tabular}{lrrrr}
  \toprule
$m$ &  $n$ & PPS & SRS & $\sigma_\text{loo}$ \\ 
  \midrule
10   & 100 & 3.6 & 16 & 8 \\ 
    & 1000 & 2.7 & 361 & 22 \\ 
    & 10000 & 3.8 & 9351 & 72 \\ 
    & 100000 & 17.5 & 19721 & 225 \\ 
  \midrule   
100  & 100 & 1.2 & 10 & 8 \\ 
    & 1000 & 1.3 & 84 & 22 \\ 
    & 10000 & 1.4 & 1144 & 72 \\ 
    & 100000 & 6.2 & 5894 & 225 \\ 
   \bottomrule
\end{tabular}
\caption{Standard errors,  SE($\widehat{\text{elpd}}_\text{loo}$), for PPS and SRS subsampling in relationship with $\sigma_\text{loo}$. The result are based on MCMC draws and $\hat{\theta}$ is the posterior mean for the parameters. PPS is subsampling proportional to  $-\log(p(y_i|\hat{\theta}))$, and SRS is simple random sampling.} 
\label{tab:hh_nsize}
\end{table}

Table \ref{tab:hh_nsize} shows empirical results on the scaling characteristics of the proposed method. We can see that for the PPS estimator, as the size of the data, $n$, increases, the variance of the estimator is more or less constant. Using a SRS sampling scheme, on the other hand, clearly show that to estimate the total $\text{elpd}_\text{loo}$, we would need to increase the sample size, $m$, as the number of observations, $n$, increases.

\subsection{Hierarchical models for radon measurements}

As an example of how the proposed method can be used, we exemplify with the dataset of \cite{lin1999analysis}, used as an example of hierarchical modeling by \citet{gelman2006data}.\footnote{We base our example and data on the Stan case study by Chris Fonnesbeck at \url{https://mc-stan.org/users/documentation/case-studies/radon.html}} The data make up a total of 12~573 home radon measurements in a total of 386 counties with a different number of observations per county. This example is enlightening for a number of reasons. First, it is large enough to actually take some computing time to analyze, but is still small enough so we can use MCMC to compute the full data $\text{elpd}_\text{loo}$ as gold standard. The models used here are also both regular and singular, showing the usability in a broader class of models. Finally, this example also shows how we can mix different approximation techniques for different models when doing model comparisons.

We compare seven different linear models of predicting the log radon levels in individual houses based on floor measurements and county uranium levels. The seven models are a pooled simple linear model (model 1), a non-pooled model with one intercept estimated per county (model 2), a partially pooled model with a hierarchical mean parameter per county (model 3),
a variable intercept model per county (model 4),
a variable slope model per county (model 5),
a variable intercept and slope model (model 6),
and finally a model with a county level features and county level intercepts
using the log uranium level in the county. We use vague priors based on the Stan prior choice recommendations\footnote{See \url{https://github.com/stan-dev/stan/wiki/Prior-Choice-Recommendations}} with $N(0,10)$ priors on regression coefficients and intercepts and half-$N(0,1)$ for variance parameters. We ran all models using Laplace, ADVI(FR), ADVI(MF) and MCMC. We ran the ADVI approximations for 100~000 iterations and for MCMC we use the Stan standard dynamic HMC algorithm using 2 chains and 500 warmup iterations per chain. For all models we used 1000 posterior samples to compute the $\text{elpd}_\text{loo}$. 

\begin{table}[ht]
\centering
\begin{tabular}{rrrr}
  \toprule
$M$ & Laplace & ADVI(FR) & ADVI(MF) \\ 
  \midrule
  1 & 0.24 & 0.23 & 0.34 \\ 
  2 & 0.93 & 5.11 & 0.45 \\ 
  3 & 1.44 & 4.19 & 0.28 \\ 
  4 & 2.05 & 6.99 & 0.71 \\ 
  5 & - & 5.62 & 1.04 \\ 
  6 & - & 10.99 & 2.98 \\ 
  7 & 1.70 & 7.39 & 0.89 \\ 
   \bottomrule
\end{tabular}
\caption{Posterior $\hat{k}$ values for the different radon models and the model approximations. Laplace was not possible for model 5 and 6.} 
\label{tab:post_k_hat}
\end{table}

Table \ref{tab:post_k_hat} show the $\hat{k}$ values for different approximations for the different models. We can see that for the simplest model (1) we get a good posterior approximation with just Laplace approximation, but as the models become more complex (and singular) we need better approximation techniques such as ADVI. We can see that ADVI (mean-field) in general perform well and can be used for inference in many models, even though more complex models (such as model 4-7) is not approximated sufficiently well using the mean-field approximation. ADVI full-rank, again, have problems due to the larger number of parameters in the more complex models. 

\begin{table}[ht]
\centering
\begin{tabular}{rlrrrr}
  \toprule
$M$ & Method & $\widehat{\text{elpd}}_\text{loo}$ & $SE$ & $\text{elpd}_\text{loo}$ & $\text{elpd}_\text{kcv}$ \\ 
  \midrule
1 &  Laplace & -18560 & 0.3 & -18560 & -18561 \\ 
   & ADVI(FR) & -18562 & 1.0 & -18559 & -18560 \\ 
   & ADVI(MF) & -18564 & 2.1 & -18559 & -18558 \\ 
   & MCMC & -18559 & 0.3 & -18560 & -18559 \\ 
  \midrule   
  2 &  Laplace & -17058 & 31.9 & -17049 & -17142 \\ 
   & ADVI(MF) & -17068 & 50.1 & -17059 & -17105 \\ 
   & MCMC & -17064 & 30.7 & -17067 & -17110 \\ 
  \midrule   
  3 &  Laplace & -17035 & 33.0 & -17017 & -17117 \\ 
   & ADVI(MF) & -17097 & 20.7 & -17090 & -17090 \\ 
   & MCMC & -17068 & 17.4 & -17085 & -17102 \\
  \midrule   
  4 &  Laplace & -17003 & 66.6 & -16866 & -17057 \\ 
   & ADVI(MF) & -16990 & 19.8 & -17013 & -17034 \\ 
   & MCMC & -17043 & 19.5 & -17022 & -17049 \\ 
  \midrule   
  5 & ADVI(MF) & -18223 & 37.6 & -18225 & -18285 \\ 
   & MCMC & -18295 & 52.2 & -18253 & -18308 \\ 
  \midrule   
  6 & ADVI(MF) & -16656 & 90.8 & -16603 & -16869 \\ 
   & MCMC & -16835 & 63.2 & -16798 & -16864 \\ 
    \midrule
  7 &  Laplace & -17096 & 45.8 & -17063 & -17140 \\ 
   & ADVI(MF) & -16996 & 25.4 & -16957 & -17035 \\ 
   & MCMC & -17126 & 27.3 & -17136 & -17050 \\ 
   \bottomrule
\end{tabular}
\caption{The estimated $\widehat{\text{elpd}}_\text{loo}$ using a subsample of size $m=500$ and its standard error (SE). The full $\text{elpd}_\text{loo}$ based on all observations is also included as well as $\text{elpd}_\text{kcv}$, an estimation of the $\text{elpd}$ using 10-fold cross-validation. The $\sigma_\text{loo} \approx 90$ for all approximations and models. Less than 1\% of the observation have problematic $\hat{k}$ using MCMC, making it a good gold standard.} 
\label{tab:elpd_radon}
\end{table}

Based on these approximate posteriors we can analyze the $\text{elpd}_\text{loo}$ for the different models. Table \ref{tab:elpd_radon} shows the $\text{elpd}_\text{loo}$ and an estimate, $\widehat{\text{elpd}}_\text{loo}$, based on a subsample of size 500. As a comparison we also compute $\text{elpd}_\text{kcv}$, computing an estimate of $\text{elpd}$ using a 10-fold cross-validation scheme (without bias correction). For the simple baseline model 1, we can use Laplace approximation and a subsample to estimate the $\text{elpd}_\text{loo}$ in roughly 2 seconds with a sufficient precision for most purposes. Using MCMC and computing the full elpd take roughly 35 seconds for this medium-sized dataset. 

Table \ref{tab:elpd_radon} also shows that ADVI (mean-field) work well both for regular and singular models. Using ADVI(MF) for the singular models 3 and 4, where the $\hat{k}$ values indicating a good posterior approximation, the approach works really well. We can also see that the $\hat{k}$ diagnostic works well as an indicator. The Laplace and ADVI(MF) approximations with high $\hat{k}$ values can be quite off, see model 7 for an example.

The results of Table \ref{tab:elpd_radon} also give us an idea of how the subsampling can be used. By comparing the SE of our estimates with $\sigma_\text{loo}$, that is roughly 90 for all models, we see how far a subsample with 500 observations takes us. For most models, our SE is small enough to help us decide between models, while for the more complex models. The precision needed depends on the specific use case and if we need better precision we can simply add more subsamples to get the precision needed.

If we study Table \ref{tab:elpd_radon} we see that using ADVI(MF), Laplace and a subsample of size 500 we can get quite far comparing these models. We could quickly rule out model 1 and 5, but where we would need to use MCMC for model 5, due to the high $\hat{k}$ for the ADVI approximations. Model 4 and 6 are the most promising but we need to estimate the models using MCMC due to the high $\hat{k}$ values for the ADVI approximations. Although, based on just the subsample, we can see that model 6, the variable intercept and slope model, seem to be the most promising model for this data. Comparing the fully computed $\text{elpd}_\text{loo}$ for the different models we could compute the difference in $\text{elpd}_\text{loo}$ between model 6 and 4 to ~220 with a standard error of 26, clearly indicating that model 6 is the one to prefer in this situation. Using 10-fold cross-validation (see $\text{elpd}_\text{kcv}$), we arrive at a similar result, but at the cost of re-estimating the model 10 times.

\section{Conclusions}

In this work we solve the two major hurdles for using leave-one-out cross-validation for large data, namely using posterior approximations to estimate the $\text{elpd}_\text{loo}$ for individual observations and efficient subsampling. We prove the consistency in $n$ and also show that for regular models we have consistency also for common posterior approximations such as Laplace and ADVI, even for mean-field ADVI in situations with correlated posteriors, making the results promising for large-scale model evaluations. Finally, our proposed method also comes with diagnostics to assess if the quality of the subsampling and posterior approximations. We can use the $\hat{k}$ diagnostic to asses the posterior approximations and $v(\widehat{\text{elpd}}_\text{loo})$, the variance of the HH estimator, to give us a good measure of the uncertainty due to subsampling.


\newpage



\bibliography{references}

\begin{thebibliography}{27}
\providecommand{\natexlab}[1]{#1}
\providecommand{\url}[1]{\texttt{#1}}
\expandafter\ifx\csname urlstyle\endcsname\relax
  \providecommand{\doi}[1]{doi: #1}\else
  \providecommand{\doi}{doi: \begingroup \urlstyle{rm}\Url}\fi

\bibitem[Azevedo-Filho \& Shachter(1994)Azevedo-Filho and
  Shachter]{azevedo1994laplace}
Azevedo-Filho, Adriano and Shachter, Ross~D.
\newblock Laplace's method approximations for probabilistic inference in belief
  networks with continuous variables.
\newblock In \emph{Uncertainty Proceedings 1994}, pp.\  28--36. Elsevier, 1994.

\bibitem[Bengio \& Grandvalet(2004)Bengio and Grandvalet]{bengio2004no}
Bengio, Yoshua and Grandvalet, Yves.
\newblock No unbiased estimator of the variance of k-fold cross-validation.
\newblock \emph{Journal of machine learning research}, 5\penalty0
  (Sep):\penalty0 1089--1105, 2004.

\bibitem[Bernardo(1979)]{bernardo1979expected}
Bernardo, Jos{\'e}~M.
\newblock Expected information as expected utility.
\newblock \emph{the Annals of Statistics}, pp.\  686--690, 1979.

\bibitem[Bernardo \& Smith(1994)Bernardo and Smith]{bernardo1994bayesian}
Bernardo, Jos{\'e}~M and Smith, Adrian~FM.
\newblock \emph{Bayesian theory}.
\newblock IOP Publishing, 1994.

\bibitem[Blei et~al.(2017)Blei, Kucukelbir, and McAuliffe]{blei2017variational}
Blei, David~M, Kucukelbir, Alp, and McAuliffe, Jon~D.
\newblock Variational inference: A review for statisticians.
\newblock \emph{Journal of the American Statistical Association}, 112\penalty0
  (518):\penalty0 859--877, 2017.

\bibitem[Box(1976)]{box1976science}
Box, George~EP.
\newblock Science and statistics.
\newblock \emph{Journal of the American Statistical Association}, 71\penalty0
  (356):\penalty0 791--799, 1976.

\bibitem[Carpenter et~al.(2017)Carpenter, Gelman, Hoffman, Lee, Goodrich,
  Betancourt, Brubaker, Guo, Li, and Riddell]{carpenter2017stan}
Carpenter, Bob, Gelman, Andrew, Hoffman, Matthew~D, Lee, Daniel, Goodrich, Ben,
  Betancourt, Michael, Brubaker, Marcus, Guo, Jiqiang, Li, Peter, and Riddell,
  Allen.
\newblock Stan: A probabilistic programming language.
\newblock \emph{Journal of statistical software}, 76\penalty0 (1), 2017.

\bibitem[Cochran(1977)]{cochran77}
Cochran, William~G.
\newblock \emph{Sampling Techniques, 3rd Edition.}
\newblock John Wiley, 1977.

\bibitem[Gelfand(1996)]{gelfand1996model}
Gelfand, Alan~E.
\newblock Model determination using sampling-based methods.
\newblock \emph{Markov chain Monte Carlo in practice}, pp.\  145--161, 1996.

\bibitem[Gelman \& Hill(2006)Gelman and Hill]{gelman2006data}
Gelman, Andrew and Hill, Jennifer.
\newblock \emph{Data analysis using regression and multilevel/hierarchical
  models}.
\newblock Cambridge university press, 2006.

\bibitem[Geweke(1989)]{geweke1989}
Geweke, John.
\newblock {B}ayesian inference in econometric models using {M}onte {C}arlo
  integration.
\newblock \emph{Econometrica: Journal of the Econometric Society}, pp.\
  1317--1339, 1989.

\bibitem[Hansen \& Hurwitz(1943)Hansen and Hurwitz]{hansen1943}
Hansen, Morris~H. and Hurwitz, William~N.
\newblock On the theory of sampling from finite populations.
\newblock \emph{The Annals of Mathematical Statistics}, 14\penalty0
  (4):\penalty0 333--362, 12 1943.

\bibitem[Jordan et~al.(1999)Jordan, Ghahramani, Jaakkola, and
  Saul]{jordan1999introduction}
Jordan, Michael~I, Ghahramani, Zoubin, Jaakkola, Tommi~S, and Saul, Lawrence~K.
\newblock An introduction to variational methods for graphical models.
\newblock \emph{Machine learning}, 37\penalty0 (2):\penalty0 183--233, 1999.

\bibitem[Kucukelbir et~al.(2017)Kucukelbir, Tran, Ranganath, Gelman, and
  Blei]{kucukelbir2017automatic}
Kucukelbir, Alp, Tran, Dustin, Ranganath, Rajesh, Gelman, Andrew, and Blei,
  David~M.
\newblock Automatic differentiation variational inference.
\newblock \emph{The Journal of Machine Learning Research}, 18\penalty0
  (1):\penalty0 430--474, 2017.

\bibitem[Lin et~al.(1999)Lin, Gelman, Price, and Krantz]{lin1999analysis}
Lin, Chia-yu, Gelman, Andrew, Price, Phillip~N, and Krantz, David~H.
\newblock Analysis of local decisions using hierarchical modeling, applied to
  home radon measurement and remediation.
\newblock \emph{Statistical Science}, pp.\  305--328, 1999.

\bibitem[Merkle et~al.(2018)Merkle, Furr, and Rabe-Hesketh]{merkle2018bayesian}
Merkle, EC, Furr, D, and Rabe-Hesketh, S.
\newblock Bayesian model assessment: Use of conditional vs marginal
  likelihoods.
\newblock \emph{arXiv preprint arXiv:1802.04452}, 2018.

\bibitem[Robert(1996)]{robert1996intrinsic}
Robert, Christian~P.
\newblock Intrinsic losses.
\newblock \emph{Theory and decision}, 40\penalty0 (2):\penalty0 191--214, 1996.

\bibitem[{Stan Development Team}(2018)]{standev2018stancore}
{Stan Development Team}.
\newblock {The Stan Core Library}, 2018.
\newblock URL \url{http://mc-stan.org/}.
\newblock Version 2.18.0.

\bibitem[Vehtari \& Ojanen(2012)Vehtari and Ojanen]{vehtari2012survey}
Vehtari, Aki and Ojanen, Janne.
\newblock A survey of {Bayesian} predictive methods for model assessment,
  selection and comparison.
\newblock \emph{Statistics Surveys}, 6:\penalty0 142--228, 2012.

\bibitem[Vehtari et~al.(2015)Vehtari, Gelman, and Gabry]{vehtari2015pareto}
Vehtari, Aki, Gelman, Andrew, and Gabry, Jonah.
\newblock Pareto smoothed importance sampling.
\newblock \emph{arXiv preprint arXiv:1507.02646}, 2015.

\bibitem[Vehtari et~al.(2017)Vehtari, Gelman, and Gabry]{vehtari2017practical}
Vehtari, Aki, Gelman, Andrew, and Gabry, Jonah.
\newblock Practical {Bayesian} model evaluation using leave-one-out
  cross-validation and {WAIC}.
\newblock \emph{Statistics and Computing}, 27\penalty0 (5):\penalty0
  1413--1432, 2017.

\bibitem[Vehtari et~al.(2018)Vehtari, Gelman, Gabry, Yao, Piironen, and
  Goodrich]{vehtari2018loo}
Vehtari, Aki, Gelman, Andrew, Gabry, Jonah, Yao, Yuling, Piironen, Juho, and
  Goodrich, Ben.
\newblock loo: {E}fficient leave-one-out cross-validation and {WAIC} for
  {B}ayesian models.
\newblock \emph{{R} package version 2.0.0}, 2018.

\bibitem[Walker(1977)]{walker1977efficient}
Walker, Alastair~J.
\newblock An efficient method for generating discrete random variables with
  general distributions.
\newblock \emph{ACM Transactions on Mathematical Software (TOMS)}, 3\penalty0
  (3):\penalty0 253--256, 1977.

\bibitem[Walker(1969)]{walker1969asymptotic}
Walker, Andrew~M.
\newblock On the asymptotic behaviour of posterior distributions.
\newblock \emph{Journal of the Royal Statistical Society. Series B
  (Methodological)}, pp.\  80--88, 1969.

\bibitem[Wang \& Blei(2018)Wang and Blei]{wang2018frequentist}
Wang, Yixin and Blei, David~M.
\newblock Frequentist consistency of variational {Bayes}.
\newblock \emph{Journal of the American Statistical Association}, \penalty0
  (just-accepted):\penalty0 1--85, 2018.

\bibitem[Watanabe(2010)]{watanabe2010asymptotic}
Watanabe, Sumio.
\newblock Asymptotic equivalence of {Bayes} cross validation and widely
  applicable information criterion in singular learning theory.
\newblock \emph{Journal of Machine Learning Research}, 11\penalty0
  (Dec):\penalty0 3571--3594, 2010.

\bibitem[Yao et~al.(2018)Yao, Vehtari, Simpson, and Gelman]{yao18a}
Yao, Yuling, Vehtari, Aki, Simpson, Daniel, and Gelman, Andrew.
\newblock Yes, but did it work?: Evaluating variational inference.
\newblock In \emph{Proceedings of the 35th International Conference on Machine
  Learning}, pp.\  5581--5590, 2018.

\end{thebibliography}
\bibliographystyle{icml2018}


\newpage

\section*{Appendix}

\section*{Proof of Proposition 1}
\label{proofs}


A generic Bayesian model is considered; a sample $(y_1,y_2,\ldots,y_n)$, $y_i \in \mathcal{Y} \subseteq \RR$, is drawn from a true density $p_t = p(\cdot|\theta_0)$ for some true parameter $\theta_0$. The parameter $\theta_0$ is assumed to be drawn from a prior $p(\theta)$ on the parameter space $\Theta$, which we assume to be an open and bounded subset of $\RR^d$.

\medskip

A number of conditions are used. They are as follows.

\begin{itemize}
  \item[(i)] the likelihood $p(y|\theta)$ satisfies that there is a function $C:\mathcal{Y} \ra \RR_+$, such that $\E_{y \sim p_t}[C(y)^2] < \iy$ and such that for all $\theta_1$ and $\theta_2$, $|p(y|\theta_1)-p(y|\theta_2)| \leq C(y)p(y|\theta_2)\n\theta_1-\theta_2\n$.
   \item[(ii)] $p(y|\theta)>0$ for all $(y,\theta) \in \mathcal{Y} \times \Theta$,
   \item[(iii)] There is a constant $M<\iy$ such that $p(y|\theta) < M$ for all $(y,\theta)$,
  \item[(iv)] all assumptions needed in the Bernstein-von Mises (BvM) Theorem \cite{walker1969asymptotic},
  \item[(v)] for all $\theta$, $\int_{\mathcal{Y}}(-\log p(y|\theta))p(y|\theta)dy < \iy$.
\end{itemize}

{\bf Remarks.} 
\begin{itemize}
\item There are alternatives or relaxations to (i) that also work. One is to assume that there is an $\alpha>0$ and $C$ with $\E_y[C(y)^2]<\iy$  such that $|p(y|\theta_1)-p(y|\theta_2)| \leq C(y)p(y|\theta_2)\n\theta_1-\theta_2\n^\alpha$. 
There are many examples when (i) holds, e.g.\ when $y$ is normal, Laplace distributed or Cauchy distributed with $\theta$ as a one-dimensional location parameter.

\item The assumption that $\Theta$ is bounded will be used solely to draw the conclusion that $\E_{y,\theta}\n \theta-\theta_0\n \ra 0$ as $n \ra \iy$, where $y$ is the sample and $\theta$ is either distributed according to the true posterior (which is consistent by BvM) or according to a consistent approximate posterior. The conclusion is valid by the definition of consistency and the fact that the boundedness of $\Theta$ makes $\n\theta-\theta_0\n$ a bounded function of $\theta$. If it can be shown by other means for special cases that $\E_{y,\theta}\n\theta-\theta_0\n \ra 0$ despite $\Theta$ being unbounded, then our results also hold.
    
\item We can (and will) without loss of generality assume that $M=1/2$ is sufficient in (iii), for if not then simply transform data and consider $z_i=2My_i$ instead of $y_i$.

\end{itemize}

\medskip

The main quantity of interest is the mean expected log pointwise predictive density, which we want to use for model evaluation and comparison.

\begin{definition}[$\overline{\text{elpd}}$]
\label{definition}
The \emph{mean expected log pointwise predictive density} for a model $p$ is defined as
\[
\overline{\text{elpd}} = \int p_t(x) \log p(x) \, dx
\]

where $p_t(x) = p(x|\theta_0)$ is the \emph{true} density at a new unseen observation $x$ and $\log p(x)$ is the log predictive density for observation $x$.

\end{definition}

We estimate $\overline{\text{elpd}}$ using {\em leave-one-out cross-validation (loo)}.

\begin{definition}[Leave-one-out cross-validation]
\label{loo}
The loo estimator $\overline{\text{elpd}}_{loo}$ is given by

\begin{equation}
\label{elpdloo}
\overline{\text{elpd}}_{loo} = \frac{1}{n} \sum^n_{i = 1} \log p(y_i|y_{-i}),
\end{equation}

where $p(y_i|y_{-i}) = \int p(y_i|\theta)p(\theta|y_{-i})d\theta$.

\end{definition}

To estimate $\overline{\text{elpd}}_{loo}$ in turn, we use importance sampling and the Hansen-Hurwitz estimator. Definitions follow.

\begin{definition}
\label{hhloo}
The Hansen-Hurwitz estimator is given by

\[
\widehat{\overline{\text{elpd}}}_\text{loo}(m,q) = \frac{1}{m} \frac{1}{n} \sum_{j=1}^{m}  \frac{1}{\tilde{\pi}_j} \log \hat{p}(y_j|y_{-j})
\]
where $\tilde{\pi}_i$ is the probability of subsampling observation $i$, $\log \hat{p}(y_i|y_{-i})$ is the (self-normalized) importance sampling estimate of $\log p(y_i|y_{-i})$ defined as

\[
\log \hat{p}(y_i|y_{-i}) = \log\left( \frac{\frac{1}{S} \sum_{s=1}^S p(y_i|\theta_s) r(\theta_s)}{\frac{1}{S} \sum_{s=1}^S r(\theta_s)} \right)\,,
\]

where

\begin{align*}
    r(\theta_s) = & \frac{p(\theta_s|y_{-i})}{p(\theta_s|y)} \frac{p(\theta_s|y)}{q(\theta_s|y)}\\
    \propto & \frac{1}{p(y_i|\theta_s)} \frac{p(\theta_s|y)}{q(\theta_s|y)}
\end{align*}

and where $q(\theta|y)$ is an approximation of the posterior distribution, $\theta_s$ is a sample from the approximate posterior distribution $q(\theta|y)$ and S is the total posterior sample size.

\end{definition}

\begin{proposition} \label{pa}
Let the subsampling size $m$ and the number of posterior draws $S$ be fixed at arbitrary integer numbers, let the sample size $n$ grow, assume that (i)-(vi) hold and let $q=q_n(\cdot|y)$ be any consistent approximate posterior. Write $\hat{\theta}_q = \arg\max\{q(\theta): \theta \in \Theta\}$ and assume further that $\hat{\theta}_q$ is a consistent estimator of $\theta_0$. Then

\[|\widehat{\overline{\text{elpd}}}_\text{loo}(m, q)-\overline{\text{elpd}}_{loo}| \ra 0\]
in probability as $n \ra \iy$ for any of the following choices of $\pi_i$, $i=1,\ldots,n$.

\begin{itemize}
  \item[(a)] $\pi_i= -\log p(y_i|y)$,
  \item[(b)] $\pi_i= -\E_y[\log p(y_i|y)]$,
  \item[(c)] $\pi_i= - \E_{\theta \sim q}[\log p(y_i|\theta)]$,
  \item[(d)] $\pi_i = -\log p(y_i|\E_{\theta \sim q}[\theta])$,
  \item[(e)] $\pi_i = -\log p(y_i|\hat{\theta}_q)$.
\end{itemize}

\end{proposition}

{\em Remark.} By the variational BvM Theorems of Wang and Blei, \cite{wang2018frequentist}, $q$ can be taken to be either $q_{Lap}$, $q_{MF}$ or $q_{FR}$, i.e.\ the approximate posteriors of the Laplace, mean-field or full-rank variational families respectively in Proposition \ref{pa}, provided that one adopts the mild conditions in their paper.

The proof of Proposition \ref{pa} will be focused on proving (a) and then (b)-(e) will follow easily. We begin with the following key lemma.

\begin{lemma} \label{la}

With all quantities as defined above,
\begin{equation} \label{ea}
\E_{y \sim p_t}|\pi_i - \log p(y_i|\theta_0)| \ra 0,
\end{equation}
with any of the definitions (a)-(e) of $\pi_i$ of Proposition \ref{pa}.
Furthermore,
\begin{equation} \label{eb}
\E_{y \sim p_t}|\log p(y_i|y_{-i}) - \log p(y_i|\theta_0)| \ra 0,
\end{equation}
and
\begin{equation} \label{ec}
\E_{y \sim p_t}|\log \hat{p}(y_i|y) - \log p(y_i|\theta_0)| \ra 0.
\end{equation}
as $n \ra \iy$.

\end{lemma}

\begin{proof}

To avoid burdening the notation unnecessarily, we write throughout the proof $\E_y$ for $\E_{y \sim p_t}$. For now, we also write $\E_\theta$ as shorthand for $\E_{\theta \sim p(\cdot|y_{-i})}$. Recall that $x_{+} = \max(x,0) = ReLU(x)$.

Hence
\begin{align*}
& \E_{y}\left[\left(\log\frac{p(y_i|y_{-i})}{p(y_i|\theta_0)}\right)_{\foo}\right] \\
&= \E_{y}\left[\left(\log\frac{\E_{\theta}[p(y_i|\theta)]}{p(y_i|\theta_0)}\right)_{\foo}\right] \\
&\leq \E_{y}\left[\log \left(1+\frac{\E_{\theta}\left[C(y_i)p(y_i|\theta_0)\n\theta-\theta_0\n\right]}{p(y_i|\theta_0)}\right)\right] \\
&\leq \E_{y,\theta}[C(y_i) \n\theta-\theta_0\n] \\
&\leq \left(\E_{y_i}[C(y_i)^2]\E_{y,\theta}\left[\n\theta-\theta_0\n^2\right]\right)^{1/2} \\
&\ra 0 \mbox{ as } n \ra \iy.
\end{align*}

Here the first inequality follows from condition (i) and the second inequality from the fact that $\log(1+x)<x$ for $x \geq 0$. The third inequality is Schwarz inequality. The limit conclusion follows from the consistency of the posterior $p(\cdot|y_{-i})$ and the definition of weak convergence, since $\n\theta-\theta_0\n^2$ is a continuous bounded function of $\theta$ (recall that $\Theta$ is bounded) and that the first factor is finite by condition (i).

For the reverse inequality,
\begin{align*}
& \E_{y}\left[\left(\log\frac{p(y_i|\theta_0)}{p(y_i|y_{-i})}\right)_{\foo}\right] \\
&= \E_{y}\left[\left(\log\E_{\theta}\left[\frac{p(y_i|\theta_0)]}{p(y_i|\theta)}\right]\right)_{\foo}\right] \\
&\leq \E_{y}\left[\log \left(1+\E_{\theta}\left[\frac{C(y_i)p(y_i|\theta)\n\theta-\theta_0\n}{p(y_i|\theta)}\right]\right)\right] \\
&\leq \left(\E_{y_i}[C(y_i)^2] \E_{y,\theta}\left[\n\theta-\theta_0\n^2\right]\right)^{1/2} \\
&\ra 0 \mbox{ as } n \ra \iy.
\end{align*}

This proves (\ref{eb}) and an identical argument proves (\ref{ea}) for $\pi_i=p(y_i|y)$. 

For $\pi_i=-\E_y[\log p(y_i|y)]$, note first that 
\begin{align*}
& \E_y\left|\E_y[\log p(y_i|y)]-\E_y[\log p(y_i|y_{-i})]\right| \\ 
&= \left|\E_y[\log p(y_i|y) - \log p(y_i|y_{-i})]\right| \\
&\leq \E_y\left|\log p(y_i|y) - \log p(y_i|y_{-i})]\right|
\end{align*}

which goes to $0$ by (\ref{eb}) and (a). Hence we can replace $\pi_i=-\E[\log p(y_i|y)]$ with $\pi_i=-\E[\log p(y_i|y_{-i})]$ when proving (b). To that end, observe that
\begin{align*}
& \left(\E_y[\log p(y_i|y_{-i})]-\log p(y_i|\theta_0)\right)_+ \\
&= \left(\E_{y_i}\left[\E_{y_{-i}}\left[\log\frac{p(y_i|y_{-i})}{p(y_i|\theta_0)}\right]\right]\right)_{\foo} \\
&\leq \E_y \left[ \left( \log\frac{p(y_i|y_{-i})}{p(y_i|\theta_0)}\right)_{\foo} \right].
\end{align*}
where the inequality is Jensen's inequality used twice on the convex function $x \ra x_+$. Now everything is identical to the proof of (\ref{eb}) and the reverse inequality is analogous.

The other choices of $\pi_i$ follow along very similar lines. For $\pi_i=-\log p(y_i|\hat{\theta}_q)$, we have on mimicking the above that
\begin{align*}
& \E_{y}\left[\left(\log\frac{p(y_i|\hat{\theta}_q)}{p(y_i|\theta_0)}\right)_{\foo}\right] \\
&\leq \left(\E_{y_i}[C(y_i)^2]\E_{y}\left[\n\hat{\theta}_q-\theta_0\n^2\right]\right)^{1/2}
\end{align*}
and $\E_{y}[\n\hat{\theta}_q-\theta_0\n^2] \ra 0$ as $n \ra \iy$ by the assumed consistency of $\hat{\theta}_q$. The reverse inequality is analogous and (\ref{ea}) for $\pi_i=p(y_i|\hat{\theta}_q)$ is established.

For the case $\pi_i=-\log p(y_i|\E_{\theta \sim q}\theta)$, the analogous analysis gives
\begin{align*}
& \E_{y}\left[\left(\log\frac{p(y_i|\E_{\theta \sim q}\theta)}{p(y_i|\theta_0)}\right)_{\foo}\right] \\
&\leq \E_{y_i}[C(y_i)^2]\E_y[\n\E_{\theta \sim q}\theta - \theta_0\n^2].
\end{align*}
Since $x \ra \n x-\theta_0\n^2$ is convex, the second factor on the right hand side is bounded by $\E_{y,\theta \sim q}[\n \theta-\theta_0\n^2]$ which goes to 0 by the consistency of $q$ and the boundedness of $\Theta$. The reverse inequality is again analogous.

Finally for $\pi_i = -\E_{\theta \sim q}[\log p(y_i|\theta)]$,
\begin{align*}
& \E_{y}\left[\left( \E_{\theta \sim q}[\log p(y_i|\theta)] - \log p(y_i|\theta_0) \right)_{\foo} \right] \\
&= \E_y\left[ \left( \E_{\theta \sim q}\left[\log\frac{p(y_i|\theta)}{p(y_i|\theta_0)}\right] \right)_{\foo} \right] \\
&\leq \E_{y,\theta \sim q}\left[ \left(\log\frac{p(y_i|\theta)}{p(y_i|\theta_0)}\right)_{\foo} \right] \\
&\leq \left(\E_{y_i}[C(y_i)^2]\E_{y,\theta \sim q}[\n\theta-\theta_0\n^2]\right)^{1/2} \ra 0
\end{align*}
as $n \ra \iy$ by the consistency of $q$. Here the first inequality is Jensen's inequality applied to $x \ra x_+$ and the second inequality follows along the same lines as before.

For (\ref{ec}), write $r'(\theta_s) = r(\theta_s)/\sum_{j=1}^{S}r(\theta_j)$ for the random weights given to the individual $\theta_s$:s in the expression for $\hat{p}(y_i|y_{-i})$. Then we have, with $\theta=(\theta_1,\ldots,\theta_S)$ chosen according to $q$,
\begin{align*}
& \E_y \left[ \left(\log \frac{\hat{p}(y_i|y_{-i})}{p(y_i|\theta_0)}\right)_{\foo}\right] \\
&= \E_{y,\theta}\left[\left(\log\frac{\sum_{s=1}^{S}r'(\theta_s)p(y_i|\theta_s)}{p(y_i|\theta_0)}\right)_{\foo}\right] \\
&\leq \E_{y,\theta}\left[\log\left(1+\frac{\sum_{s=1}^{S}r'(\theta_s)|p(y_i|\theta_s)-p(y_i|\theta_0)|}{p(y_i|\theta_0)}\right)\right] \\
&\leq \E_{y,\theta}\left[\log\left(1+C(y_i)\sum_{s=1}^{S}r'(\theta_s)\n\theta_s-\theta_0\n\right)\right] \\
&\leq \E_{y,\theta}\left[\log\left(1+C(y_i)\sum_{s=1}^{S}\n\theta_s-\theta_0\n\right)\right] \\
&\leq \E_{y,\theta}\left[C(y_i)\sum_{s=1}^{S} \n\theta_s-\theta_0\n\right] \\
&\leq \left(\E_{y_i}[C(y_i)^2]\E_{y,\theta}\left[\left(\sum_{s=1}^{S} \n\theta_s-\theta_0\n\right)^2\right]\right)^{1/2},
\end{align*}

where the second inequality is condition (i) and the limit conclusion follows from the consistency of $q$.
For the reverse inequality to go through analogously, observe that
\begin{align*}
& \frac{\left| p(y_i|\theta_0)-\sum_s r'(\theta_s)p(y_i|\theta_s) \right|}{\sum_s r'(\theta_s)p(y_i|\theta_s)} \\
&\leq \frac{\sum_s r'(\theta_s)|p(y_i|\theta_s)-p(y_i|\theta_0)|}{\sum_s r'(\theta_s)p(y_i|\theta_s)} \\
&\leq \frac{\sum_s r'(\theta_s)p(y_i|\theta_s)\n \theta_s-\theta_0 \n}{\sum_s r'(\theta_s)p(y_i|\theta_s)} \\
&\leq \max_s\n\theta_s-\theta_0\n \\
&\leq \sum_s \n \theta_s-\theta_0 \n.
\end{align*}
Equipped with this observation, mimic the above.

\end{proof}

For convenience we will write $\hat{e}:=\hat{e}_{m,q} = \widehat{\overline{\text{elpd}}}_{loo}$, which for our purposes is more usefully expressed as
\[\hat{e}=\frac{1}{n}\frac{1}{m}\sum_{i=1}^{n}\sum_{j=1}^{m}I_{ij}\frac{1}{\bar{\pi}_i}\log \hat{p}(y_i|y_{-i}),\]
where $I_{ij}$ is the indicator that sample point $y_i$ is chosen in draw $j$ for the subsample used in $\hat{e}$.
Write also
\[e=\frac{1}{n}\frac{1}{m}\sum_{i=1}^{n}\sum_{j=1}^{m}I_{ij}\frac{1}{\bar{\pi}_i}\log p(y_i|y_{-i}).\]
In other words, $e$ is the HH estimator with $\hat{p}$ replaced with $p$.

\begin{lemma} \label{lb}
  With the notation as just defined and $\pi_i=- \log p(y_i|y)$,
  \[\E|\hat{e}-e| \ra 0\]
  as $n \ra \iy$.
\end{lemma}

\begin{proof}
We have, with expectations with respect to all sources of randomness involved in $\hat{e}$ and $e$
\begin{align*}
& \E|\hat{e}-e| \\
\leq\, & \frac{1}{m}\frac{1}{n}\sum_{i=1}^{n}\sum_{j=1}^{m}\E\left[\E\left[I_{ij}\frac{1}{\bar{\pi_i}}|\log \hat{p}(y_i|y_{-i}) - \log p(y_i|y_{-i})|  \Big| y\right]\right] \\
=\, & \E\left[\frac{1}{n}\frac{1}{m}\sum_{i=1}^{n}\sum_{j=1}^{m}|\log \hat{p}(y_i|y_{-i}) - \log p(y_i|y_{-i})|\right] \\
=\, & \E|\log \hat{p}(y_i|y_{-i}) - \log p(y_i|y_{-i})|.
\end{align*}
The result now follows from (\ref{eb}), (\ref{ec}) and the triangle inequality.
\end{proof}

{\em Proof of Proposition \ref{pa}.}
As stated before, we start with a focus on (a), which means that for now we have $\pi_i = - \log p(y_i|y)$
By Lemma \ref{lb}, it suffices to prove that $|e-\overline{\text{elpd}}_{loo}| \rightarrow 0$ in probability with $\pi_i$ chosen according to any of (a)-(e). The variance of a HH estimator is well known and some easy manipulation then tells us that the conditional variance of $e$ given $y$ is given by
\[V(e) = \Var(e|y) =\frac{1}{n^2}\frac{1}{m}(S_\pi S_2-S_p^2),\]
where $S_p=\sum_{i=1}^{n}p_i$, $S_\pi = \sum_{i=1}^{n}\pi_i$ and $S_2=\sum_{i=1}^{n}(p_i^2/\pi_i)$.
We claim that for any $\delta>0$, for $n$ sufficiently large, $\Pro_y(V(e)<\delta)>1-\delta$.
To this end, observe first that
\begin{align*}
& \E_y[-\log p(y_i|y)] \\
\leq\, & \E_y[-\log p(y_i|\theta_0)] + \E_y|\log p(y_i|y) - \log p(y_i|\theta_0)| \\
\leq\,& \E_y[-\log p(y_i|\theta_0)] + \delta < \iy
\end{align*}
for sufficiently large $n$, since the first term is finite by condition (v). Let $A=A_n=\E_y[-\log p(y_i|y)]$.

Now,
\[\E_y\left[\frac{1}{n}|S_p-S_\pi|\right] = \E_y\left[\frac{1}{n}\left|\sum_{i=1}^{n}\pi_i-\sum_{i=1}^{n}p_i\right|\right] \ra 0\]
as $n \ra \iy$ by (\ref{ea}) and (\ref{eb}). Hence for arbitrary $\alpha>0$, $\Pro_y(|S_p-S_\pi| < \alpha^2 n) > 1-\alpha$ for $n$ large enough.
Also
\[\frac{p_i^2}{\pi_i} \leq \frac{(\pi_i+|p_i-\pi_i|)^2}{\pi_i} < \pi_i+4|\pi_i-p_i|\]
(the last inequality using condition (iii): $\pi_i \geq -\log(1/2) > 1/2$),
so $n^{-1}\E_y|S_\pi-S_2| \ra 0$ and so $\Pro_y(|S_p-S_2| < \alpha^2 n) > 1- \alpha$ for sufficiently large $n$.
Hence with probability exceeding $1-2\alpha$, $y$ will be such that for sufficiently large $n$,
\begin{align*}
V(e) & \leq \frac{1}{n^2} \frac{1}{m}\left((S_p+\alpha^2 n)^2-S_p^2\right) \\
& = \frac{1}{n^2} \frac{1}{m} (2\alpha^2 n S_p+\alpha^4 n^2).
\end{align*}
We had $\E_y[S_p] = An$ and Markov's inequality thus entails that $\Pro_y(S_p<An/\alpha) > 1-\alpha$. Adding this piece of information to the above, we get that with probability larger than $1-3\alpha$, $y$ will for sufficiently large $n$ be such that
\[V(e) \leq (2\alpha + \alpha^4)n^2 < 3\alpha.\]
For such $y$, Chebyshev's inequality gives
\[\Pro(|e-\E[e|y]|>\alpha^{1/2}|y) < 3\alpha^{1/2}.\]
The HH estimator is unbiased, so $\E[e|y] = \overline{\text{elpd}}_{loo}$. We get for arbitrary $\ep>0$ on taking $\alpha$ sufficiently small and $n$ correspondingly large, taking all randomness into account
\[\Pro(|e-\overline{\text{elpd}}_{loo}|>\ep)<1-\ep\]
which entails that $|e-\overline{\text{elpd}}_{loo}| \ra 0$ in probability. As observed above, this proves (a).

For the remaining parts, write $e_p$ when taking $\pi_i$ in $e$ according to statement (p) in the proposition. By (\ref{ea}),
$\E|e_p-e_a| \ra 0$ for $p=b,c,d,e$ and we are done.

\hfill $\Box$



%



\newpage

\section*{Unbiasness of using the Hansen-Hurwitz estimator}

\subsection{On the Hansen-Hurwitz estimator}
Let $\mathcal{Y} = \left\lbrace y_1, y_2, \hdots, y_N \right\rbrace$ be a set of non-negative observations, $y_i > 0$ and let $\pi= \left\lbrace \pi_1, \pi_2, \hdots, \pi_N\right\rbrace$ be a probability vector s.t. $\sum \pi_j = 1$. Furthermore, let $a_k \in \left\lbrace 1, 2, \hdots, N\right\rbrace$ be i.i.d. samples from a multinomial distribution with probabilities $\pi$, i.e. $a_k \stackrel{iid}{\sim} \text{Multinomial}\left(\pi\right)$.

We want to estimate the total
\begin{align}
	\tau = \sum_{n=1}^N y_i
\end{align}

using the Hansen-Hurwitz estimator given by
\begin{align}
	\hat{\tau} = \frac{1}{M}\sum_{m=1}^M \frac{x_m}{p_m},
\end{align}
where $x_m \equiv y_{a_m}$, $p_m \equiv \pi_{a_m}$, and $a_m \sim \text{Multinomial}\left(\pi\right)$.

We can decompose $x_m$ and $p_m$ as follows
\begin{align} 
	x_m \equiv y_{a_m} &= \sum_{j=1}^N \mathbb{I}\left[a_m = j\right] y_j \label{eq:x_i}\\
	p_m \equiv p_{a_m} &= \sum_{j=1}^N \mathbb{I}\left[a_m = j\right] \pi_j \label{eq:p_i}
\end{align}

\subsection{The Hansen-Hurwitz estimator is unbiased}
First, we will show that the HH estimator, $\hat{\tau}$, is unbiased. We have,	
\begin{align}
	\mathbb{E}\left[\hat{\tau}\right] &= \mathbb{E}\left[\frac{1}{M}\sum_{m=1}^M \frac{x_m}{p_m}\right]
	= \frac{1}{M}\sum_{m=1}^M \mathbb{E}\left[\frac{x_m}{p_m}\right]
\end{align}

Using the definitions in eq. \eqref{eq:x_i} and \eqref{eq:p_i} yields
\begin{align}
	\mathbb{E}\left[\hat{\tau}\right] &= \frac{1}{M}\sum_{m=1}^M \mathbb{E}\left[\frac{\sum_{j=1}^N \mathbb{I}\left[a_m = j\right] y_j}{\sum_{j=1}^N \mathbb{I}\left[a_m = j\right] \pi_j}\right]\nonumber\\
	&= \frac{1}{M}\sum_{m=1}^M \mathbb{E}\left[\sum_{j=1}^N \frac{y_j}{\pi_j} \mathbb{I}\left[a_m = j\right] \right]\nonumber\\
	&= \frac{1}{M}\sum_{m=1}^M \sum_{j=1}^N \frac{y_j}{\pi_j} \mathbb{E}\left[\mathbb{I}\left[a_m = j\right] \right]\nonumber\\
	&= \frac{1}{M}\sum_{m=1}^M \sum_{j=1}^N \frac{y_j}{\pi_j} \pi_j 
\end{align}
since $\pi_j = \mathbb{P}\left[a_m = j\right] = \mathbb{E}\left[\mathbb{I}\left[a_m = j\right]\right]$.

Now it follows that
\begin{align}
	\mathbb{E}\left[\hat{\tau}\right] &= \frac{1}{M}\sum_{m=1}^M \sum_{j=1}^N y_j =  \sum_{j=1}^N y_j = \tau.
\end{align}

\subsection{An unbiased estimator of $\sigma^2_\text{loo}$}
We also want to estimate the variance of the population $\mathcal{Y}$, i.e.
\begin{align}
	\sigma_y^2 = \frac{1}{N}\sum_{n=1}^N \left(y_n - \bar{y}\right)^2,
\end{align}
where $\bar{y} = \frac{1}{N}\sum y_n$.

First, we decompose the above as follows
\begin{align}
	\sigma_y^2 = \frac{1}{N}\sum_{n=1}^N y^2_n - \bar{y}^2.
\end{align}

We will consider estimators for the two terms, $\frac{1}{N}\sum_{n=1}^N y^2_n$ (1) and $\bar{y}^2$ (2), separately. First, we will show that the following is an unbiased estimate of the first term,
\begin{align}
	T_1 = \frac{1}{NM}\sum_{m=1}^M \frac{x_m^2}{p_m}.
\end{align}

 We have
\begin{align}
\mathbb{E}\left[T_1\right] =\mathbb{E}\left[\frac{1}{NM}\sum_{m=1}^M \frac{x_m^2}{p_m}\right] = \frac{1}{NM} \sum_{m=1}^M\mathbb{E}\left[ \frac{x_m^2}{p_m}\right]
\end{align}
Again, we use the representations in eq. \eqref{eq:x_i} and \eqref{eq:p_i} to get
\begin{align}
	\mathbb{E}\left[\frac{1}{NM}\sum_{m=1}^M \frac{x_m^2}{p_m}\right] & = \frac{1}{NM} \sum_{m=1}^M\mathbb{E}\left[ \frac{\sum_{j=1}^N \mathbb{I}\left[a_m = j\right] y_j^2}{\sum_{j=1}^N \mathbb{I}\left[a_m = j\right] \pi_j}\right]\nonumber\\
	&= \frac{1}{NM} \sum_{m=1}^M\mathbb{E}\left[  \sum_{j=1}^N \mathbb{I}\left[a_m = j\right]  \frac{y_j^2}{\pi_j}\right]\nonumber\\
	&= \frac{1}{NM} \sum_{m=1}^M \sum_{j=1}^N \frac{y_j^2}{\pi_j}\mathbb{E}\left[ \mathbb{I}\left[a_m = j\right]  \right]\nonumber\\
	&= \frac{1}{NM} \sum_{m=1}^M \sum_{j=1}^N \frac{y_j^2}{\pi_j}\pi_j\nonumber\\
	&= \frac{1}{N} \sum_{j=1}^N y_j^2.
\end{align}

This completes the proof of for the first term.

For the second term, we use the estimator $T_2$ given by
\begin{align}
	T_2 = &  \frac{1}{M(M-1)} \sum_{m=1}^M \left[\frac{x_m}{N p_m} - \frac{1}{N}\sum_{k=1}^M \frac{x_k}{M p_k}\right]^2 \nonumber\\ 
	& - \left[\frac{1}{N}\sum_{k=1}^M \frac{x_k}{M p_k}\right]^2.
\end{align}

We have
{\small
\begin{align}
	& \frac{1}{M(M-1)} \sum_{m=1}^M \left[\frac{x_m}{N p_m} - \sum_{k=1}^M \frac{x_k}{NM p_k}\right]^2 - \left[\sum_{k=1}^M \frac{x_k}{NM p_k}\right]^2 \nonumber\\
	%
	%
	%
	%
	%
	%
	%
	&= \frac{1}{N^2M(M-1)} \sum_{m=1}^M \frac{x^2_m}{p^2_m}  - \frac{1}{N^2M(M-1)} \left[\sum_{k=1}^M \frac{x_k}{p_k}\right]^2 
	%
\end{align}}

We consider now the expectation of the first term in the equation above
\begin{align}
	\mathbb{E}\left[\sum_{m=1}^M \frac{x^2_m}{p^2_m}\right] &= \sum_{m=1}^M \mathbb{E}\left[\frac{x^2_m}{p^2_m}\right]\nonumber\\
	&= \sum_{m=1}^M \mathbb{E}\left[\frac{\sum_{j=1}^N \mathbb{I}\left[a_m = j\right] y_j^2}{\sum_{j=1}^N \mathbb{I}\left[a_m = j\right]  \pi_j^2}\right]\nonumber\\
	&= \sum_{m=1}^M \mathbb{E}\left[\sum_{j=1}^N \mathbb{I}\left[a_m = j\right] \frac{y_j^2}{ \pi_j^2}\right]\nonumber\\
	&= \sum_{m=1}^M \sum_{j=1}^N \mathbb{E}\left[\mathbb{I}\left[a_m = j\right]\right] \frac{y_j^2}{ \pi_j^2}\nonumber\\
	&= M\sum_{j=1}^N  \frac{y_j^2}{ \pi_j}
	\end{align}

	and the second term

\begin{align}
	\mathbb{E}\left[\left[\sum_{k=1}^M \frac{x_k}{p_k}\right]^2 \right] &= \mathbb{E}\left[\sum_{k=1}^M \sum_{j=1}^M \frac{x_k}{p_k} \frac{x_j}{p_j} \right] \nonumber\\
	&= \sum_{k=1}^M \sum_{j=1}^M \mathbb{E}\left[\frac{x_k}{p_k} \frac{x_j}{p_j} \right] \nonumber\\
	&= \sum_{j\neq k}^M \mathbb{E}\left[\frac{x_k}{p_k} \frac{x_j}{p_j} \right] + \sum_{k=1}^M  \mathbb{E}\left[\frac{x_k^2}{p_k^2} \right]\nonumber\\
	&= \sum_{j\neq k}^M \mathbb{E}\left[\frac{x_k}{p_k}\right]\mathbb{E}\left[ \frac{x_j}{p_j} \right] + \sum_{k=1}^M  \sum_{j=1}^N  \frac{y_j^2}{ \pi_j}\nonumber\\
	&= \sum_{j\neq k}^M \mathbb{E}\left[\frac{x_k}{p_k}\right]\mathbb{E}\left[ \frac{x_j}{p_j} \right] + M\sum_{j=1}^N  \frac{y_j^2}{ \pi_j}\nonumber\\
	&=M(M-1) \tau^2+   M\sum_{j=1}^N  \frac{y_j^2}{ \pi_j}.
\end{align}
Substituting back, we get
{\small
\begin{align}
	&\frac{1}{M(M-1)} \sum_{m=1}^M \left[\frac{x_m}{N p_m} - \frac{1}{N}\sum_{k=1}^M \frac{x_k}{M p_k}\right]^2 - \left[\frac{1}{N}\sum_{k=1}^M \frac{x_k}{M p_k}\right]^2\nonumber\\
	= &\frac{1}{N^2M(M-1)} M\sum_{j=1}^N  \frac{y_j^2}{ \pi_j} - \nonumber \\
	& \frac{1}{N^2M(M-1)} \left[M(M-1) \tau^2+   M\sum_{j=1}^N  \frac{y_j^2}{ \pi_j}\right]\nonumber\\
	= &\frac{1}{N^2(M-1)} \sum_{j=1}^N  \frac{y_j^2}{ \pi_j} - \nonumber\\ 
	& \frac{1}{N^2(M-1)} \left[(M-1) \tau^2+   \sum_{j=1}^N  \frac{y_j^2}{ \pi_j}\right]\nonumber\\
	=& - \frac{1}{N^2(M-1)} (M-1) \tau^2 \nonumber\\
	=& - \frac{\tau^2}{N^2}  \nonumber\\
	=& - \bar{y}^2.
\end{align}}

Combining the two estimators $T_1$ and $T_2$ we have:

{\small
\begin{align}
	\mathbb{E}(T_1 + T_2)  & =\frac{1}{N} \sum^N_{j=1} y_j^2 - \bar{y}^2\nonumber\\
	 & =\sigma^2_y \nonumber
\end{align}}

Hence, we have shown that the estimator of $\sigma^2_y$ is unbiased using the sum of the estimators $T_1$ in Eq. 14 and $T_2$ in Eq. 18.

\newpage

\section*{Hierarchical models for the radon dataset}

We compare seven different models of predicting the radon levels in individual houses (indexed by $i$) by county (indexed by $j$). First we fit a pooled model (model 1)
\begin{align*}
y_{ij} & = \alpha + x_{ij} \beta + \epsilon_{ij}\\
\epsilon_{ij} & \sim N(0,\sigma_y)\\
\alpha,\beta & \sim N(0,10)\\
\sigma_y & \sim N^+(0,1)\,,
\end{align*}
where $y_{ij}$ is the log radon level in house $i$ in county $j$, $x_{ij}$ is the floor measurement and $\epsilon_{ij}$ is $N^+(0,1)$ is a truncated Normal distribution at the positive real line. We compare this to a non-pooled model (model 2),
\begin{align*}
y_{ij} & = \alpha_j + x_{ij} \beta + \epsilon_{ij}\\
\epsilon_{ij} & \sim N(0,\sigma_y)\\
\alpha_j,\beta & \sim N(0,10)\\
\sigma_y & \sim N^+(0,1)\,,
\end{align*}
a partially pooled model (model 3),
\begin{align*}
y_{ij} & = \alpha_j + \epsilon_{ij}\\
\epsilon_{ij} & \sim N(0,\sigma_y)\\
\alpha_j & \sim N(\mu_\alpha,\sigma_\alpha)\\
\mu_\alpha & \sim N(0,10)\\
\sigma_y,\sigma_\alpha & \sim N^+(0,1)\,,
\end{align*}
a variable intercept model (model 4),
\begin{align*}
y_{ij} & = \alpha_j + x_{ij} \beta + \epsilon_{ij}\\
\epsilon_{ij} & \sim N(0,\sigma_y)\\
\alpha_j & \sim N(\mu_\alpha,\sigma_\alpha)\\
\mu_\alpha, \beta & \sim N(0,10)\\
\sigma_y,\sigma_\alpha & \sim N^+(0,1)\,,
\end{align*}
a variable slope model (model 5),
\begin{align*}
y_{ij} & = \alpha + x_{ij} \beta_j + \epsilon_{ij}\\
\epsilon_{ij} & \sim N(0,\sigma_y)\\
\beta_j & \sim N(\mu_\beta,\sigma_\beta)\\
\mu_\beta, \alpha & \sim N(0,10)\\
\sigma_y,\sigma_\beta & \sim N^+(0,1)\,,
\end{align*}
a variable intercept and slope model (model 6),
\begin{align*}
    y_{ij} & = \alpha_j + x_{ij} \beta_j + \epsilon_{ij}\\
    \alpha_j & \sim N(\mu_\alpha, \sigma_\alpha)\\
    \beta_j & \sim N(\mu_\beta, \sigma_\beta)\\
    \mu_\alpha, \mu_\beta & \sim N(0, 10)\\
    \sigma_y,\sigma_\alpha,\sigma_\beta & \sim N^+(0,1)\,,
\end{align*}
and finally a model with county level covariates and county level intercepts
\begin{align*}
    y_{ij} & = \alpha_j + x_{ij} \beta_1 + u_{j} \beta_2 + \epsilon_{ij}\\
    \alpha_j & \sim N(\mu_\alpha, \sigma_\alpha)\\
    \beta,\mu_\alpha & \sim N(0, 10)\\
    \sigma_y,\sigma_\alpha & \sim N^+(0,1)\,,
\end{align*}
where $u_j$ is the log uranium level in the county. The Stan code used can be found below.

\section*{Stan models}

\subsection{Linear regression model}

\begin{lstlisting}[language=Stan,basicstyle=\footnotesize\ttfamily]
data {
  int <lower=0> N; 
  int <lower=0> D; 
  matrix [N,D] x ; 
  vector [N] y;
}
parameters {
  vector [D] b;
  real <lower=0> sigma;
}
model {
  target += normal_lpdf(y | x * b, sigma);
  target += normal_lpdf(b | 0, 1);
}


generated quantities{
  real  log_joint_density_unconstrained; 
  vector[N] log_lik;
  // Compute the log likelihoods for loo
  for (n in 1:N) {
    log_lik[n] = 
      normal_lpdf(y[n] | x[n,] * b, sigma);
  }
}
\end{lstlisting}

\newpage
\subsection{Radon pooled model (1)}
\begin{lstlisting}[language=Stan,basicstyle=\footnotesize\ttfamily]
data {
  int<lower=0> N; 
  vector[N] x;
  vector[N] y;
  int<lower=0,upper=1> holdout[N];
}

parameters {
  vector[2] beta;
  real<lower=0> sigma_y;
} 

model {
  vector[N] mu;
  
  // priors
  sigma_y ~ normal(0,1);
  beta ~ normal(0,10);

  // likelihood  
  mu = beta[1] + beta[2] * x;
  for(n in 1:N){
    if(holdout[n] == 0){
      target += 
        normal_lpdf(y[n]|mu[n],sigma_y);
    }
  }
}

\end{lstlisting}

\newpage
\subsection{Radon pooled model (2)}
\begin{lstlisting}[language=Stan,basicstyle=\footnotesize\ttfamily]
data {
  int<lower=0> N; 
  int<lower=0> J; 
  int<lower=1,upper=J> county[N];
  vector[N] x;
  vector[N] y;
  int<lower=0,upper=1> holdout[N];
} 

parameters {
  vector[J] a;
  real beta;
  real<lower=0> sigma_y;
} 

model {
  vector[N] mu;
  // Prior
  sigma_y ~ normal(0,1);
  a ~ normal(0,10);

  // Likelihood
  for(n in 1:N){
    mu[n] = beta*x[n] + a[county[n]];
    if(holdout[n] == 0){
      target += 
        normal_lpdf(y[n]|mu[n],sigma_y);
    }
  }
}

\end{lstlisting}

\newpage
\subsection{Radon partially pooled model (3)}
\begin{lstlisting}[language=Stan,basicstyle=\footnotesize\ttfamily]
data {
  int<lower=0> N; 
  int<lower=0> J; 
  int<lower=1,upper=J> county[N];
  vector[N] y;
  int<lower=0,upper=1> holdout[N];
} 
parameters {
  vector[J] a;
  real mu_a;
  real<lower=0> sigma_a;
  real<lower=0> sigma_y;
} 

model {
  vector[N] mu;
  
  // priors
  sigma_y ~ normal(0,1);
  sigma_a ~ normal(0,1);
  mu_a ~ normal(0,10);

  // likelihood  
  a ~ normal (mu_a, sigma_a);
  for(n in 1:N){
    mu[n] = a[county[n]];
    if(holdout[n] == 0){
      target += 
        normal_lpdf(y[n]|mu[n],sigma_y);
    }
  }
}
\end{lstlisting}

\newpage
\subsection{Variable intercept model (4)}
\begin{lstlisting}[language=Stan,basicstyle=\footnotesize\ttfamily]
data {
  int<lower=0> J; 
  int<lower=0> N; 
  int<lower=1,upper=J> county[N];
  vector[N] x;
  vector[N] y;
  int<lower=0,upper=1> holdout[N];
} 
parameters {
  vector[J] a;
  real beta;
  real mu_a;
  real<lower=0> sigma_a;
  real<lower=0> sigma_y;
} 

model {
  vector[N] mu;
  // Prior
  sigma_y ~ normal(0,1);
  sigma_a ~ normal(0,1);
  mu_a ~ normal(0,10);
  beta ~ normal(0,10);

  a ~ normal (mu_a, sigma_a);
  for(n in 1:N){
    mu[n] = a[county[n]] + x[n]*beta;
    if(holdout[n] == 0){
      target += 
        normal_lpdf(y[n]|mu[n],sigma_y);
    }
  }
}
\end{lstlisting}

\newpage
\subsection{Variable slope model (5)}
\begin{lstlisting}[language=Stan,basicstyle=\footnotesize\ttfamily]
data {
  int<lower=0> J; 
  int<lower=0> N; 
  int<lower=1,upper=J> county[N];
  vector[N] x;
  vector[N] y;
  int<lower=0,upper=1> holdout[N];
} 
parameters {
  real a;
  vector[J] beta;
  real mu_beta;
  real<lower=0> sigma_beta;
  real<lower=0> sigma_y;
} 

model {
  vector[N] mu;
  // Prior
  a ~ normal(0,10);
  sigma_y ~ normal(0,1);
  sigma_beta ~ normal(0,1);
  mu_beta ~ normal(0,10);

  beta ~ normal(mu_beta,sigma_beta);
  for(n in 1:N){
    mu[n] = a + x[n] * beta[county[n]];
    if(holdout[n] == 0){
      target += 
        normal_lpdf(y[n]|mu[n],sigma_y);
    }
  }
}
\end{lstlisting}

\newpage
\subsection{Variable intercept and slope model (6)}
\begin{lstlisting}[language=Stan,basicstyle=\footnotesize\ttfamily]
data {
  int<lower=0> N;
  int<lower=0> J;
  vector[N] y;
  vector[N] x;
  int county[N];
  int<lower=0,upper=1> holdout[N];
}
parameters {
  real<lower=0> sigma_y;
  real<lower=0> sigma_a;
  real<lower=0> sigma_beta;
  vector[J] a;
  vector[J] beta;
  real mu_a;
  real mu_beta;
}

model {
  vector[N] mu;
  // Prior
  sigma_y ~ normal(0,1);
  sigma_beta ~ normal(0,1);
  sigma_a ~ normal(0,1);
  mu_a ~ normal(0,10);
  mu_beta ~ normal(0,10);

  a ~ normal(mu_a, sigma_a);
  beta ~ normal(mu_beta, sigma_beta);
  for(n in 1:N){
    mu[n] = a[county[n]] + x[n]*beta[county[n]];
    if(holdout[n] == 0){
      target += 
        normal_lpdf(y[n]|mu[n],sigma_y);
    }
  }
  
}
\end{lstlisting}

\newpage
\subsection{Hierarchical intercept model (7)}
\begin{lstlisting}[language=Stan,basicstyle=\footnotesize\ttfamily]
data {
  int<lower=0> J; 
  int<lower=0> N; 
  int<lower=1,upper=J> county[N];
  vector[N] u;
  vector[N] x;
  vector[N] y;
  int<lower=0,upper=1> holdout[N];
} 
parameters {
  vector[J] a;
  vector[2] beta;
  real mu_a;
  real<lower=0> sigma_a;
  real<lower=0> sigma_y;
} 
transformed parameters {
}

model {
  vector[N] mu;
  vector[N] m;

  sigma_a ~ normal(0, 1);
  sigma_y ~ normal(0, 1);
  mu_a ~ normal(0, 10);
  beta ~ normal(0, 10);

  a ~ normal(mu_a, sigma_a);
  for(n in 1:N){
    m[n] = a[county[n]] + u[n] * beta[1];
    mu[n] = m[n] + x[n] * beta[2];
    if(holdout[n] == 0){
      target += normal_lpdf(y[n] | mu[n], sigma_y);
    }
  }
}

\end{lstlisting}

\end{document}